\documentclass[conference]{IEEEtran}
\usepackage{times}
\usepackage{amsmath}
\usepackage{amssymb}
\usepackage{todonotes}
\usepackage[numbers]{natbib}
\usepackage{multicol}
\usepackage{xcolor}
\usepackage{amsfonts}
\usepackage{booktabs}
\usepackage{multirow}
\usepackage[bookmarks=true,backref=page]{hyperref}
\usepackage{balance}
\usepackage{siunitx}
\usepackage{gensymb}
\usepackage{balance}

\usepackage{caption}
\usepackage{amsmath} 
\usepackage{amsthm}
\newtheorem*{example}{Running example}
\theoremstyle{definition}
\newtheorem{remark}{Remark}
\newtheorem{proposition}{Proposition}

\newcommand{\jcnote}[1]{{\textcolor{black}{#1}}}

\newcommand{\sagent}{s^{(i)}}
\newcommand{\sagentj}{s^{(j)}}
\newcommand{\srelative}{s^{(ij)}}

\newcommand{\oagent}{o^{(i)}}

\newcommand{\aagent}{a^{(i)}}
\newcommand{\aagentj}{a^{(j)}}
\newcommand{\aseq}{\boldsymbol{a}^{(i)}}
\newcommand{\aseqj}{\boldsymbol{a}^{(j)}}
\newcommand{\asafe}{a_{\textrm{safe}}^{(i)}}
\newcommand{\amarl}{a_{\textrm{marl}}^{(i)}}
\newcommand{\asafej}{a_{\textrm{safe}}^{(j)}}
\newcommand{\amarlj}{a_{\textrm{marl}}^{(j)}}
\newcommand{\aset}{\mathcal{A}}

\newcommand{\brt}{\mathcal{BRT}}
\newcommand{\safesetpair}{\mathcal{S}^{(ij)}}

\newcommand{\unsafeset}{\mathcal{L}^{(ij)}}

\newcommand{\timestep}{k}

\newcommand{\samplingtime}{\Delta t}
\newcommand{\distance}{\mathrm{dist}}
\newcommand{\cost}{\mathcal{C}}

\newcommand{\dsafety}{r_{\textrm{safety}}}
\newcommand{\dengagement}{r_{\textrm{conflict}}}
\newcommand{\dobservation}{r_{\textrm{obs}}}

\newcommand{\algoname}{Layered Safe MARL}

\begin{document}

\title{Resolving Conflicting Constraints in Multi-Agent Reinforcement Learning with Layered Safety}

\author{
\IEEEauthorblockN{
Jason J. Choi\IEEEauthorrefmark{1}$^{1}$,
Jasmine Jerry Aloor\IEEEauthorrefmark{1}$^{2}$,
Jingqi Li\IEEEauthorrefmark{1}$^{1}$,
Maria G. Mendoza$^{1}$, \\
Hamsa Balakrishnan\IEEEauthorrefmark{2}$^{2}$,
Claire J. Tomlin\IEEEauthorrefmark{2}$^{1}$
}\vspace{0.3cm}
\IEEEauthorblockA{$^{1}$University of California, Berkeley \quad Email: 
\{jason.choi, jingqili, maria$\_$mendoza, tomlin\}@berkeley.edu}
\IEEEauthorblockA{$^{2}$Massachusetts Institute of Technology \quad Email: 
\{jjaloor, hamsa\}@mit.edu}
\vspace{0.5em}
\IEEEauthorblockA{\IEEEauthorrefmark{1}\footnotesize Equal contributions \IEEEauthorrefmark{2}\footnotesize Equal advising}
}

\maketitle

\begin{abstract}
Preventing collisions in multi-robot navigation is crucial for deployment. This requirement hinders the use of learning-based approaches, such as multi-agent reinforcement learning (MARL), on their own due to their lack of safety guarantees. Traditional control methods, such as reachability and control barrier functions, can provide rigorous safety guarantees when interactions are limited only to a small number of robots. However, conflicts between the constraints faced by different agents pose a challenge to safe multi-agent coordination.

To overcome this challenge, we propose a method that integrates multiple layers of safety by combining MARL with safety filters. First, MARL is used to learn strategies that minimize multiple agent interactions, where multiple indicates more than two. Particularly, we focus on interactions likely to result in conflicting constraints within the engagement distance. Next, for agents that enter the engagement distance, we prioritize pairs requiring the most urgent corrective actions. Finally, a dedicated safety filter provides tactical corrective actions to resolve these conflicts. Crucially, the design decisions for all layers of this framework are grounded in reachability analysis and a control barrier-value function-based filtering mechanism.

We validate our \algoname{} framework in 1) hardware experiments using Crazyflie drones and 2) high-density advanced aerial mobility (AAM) operation scenarios, where agents navigate to designated waypoints while avoiding collisions. The results show that our method significantly reduces conflict while maintaining safety without sacrificing much efficiency (i.e., shorter travel time and distance) compared to baselines that do not incorporate layered safety. 
\href{https://dinamo-mit.github.io/Layered-Safe-MARL/}{[Project Webpage]}\footnote{ \href{https://dinamo-mit.github.io/Layered-Safe-MARL/}{Project Webpage: https://dinamo-mit.github.io/Layered-Safe-MARL/}}

\end{abstract}

\IEEEpeerreviewmaketitle

\section{Introduction}
\subsection{Motivation}

Collision-free operation is a fundamental requirement for multi-robot coordination tasks, such as formation control \cite{poonawala2014collision}, multi-robot payload transport \cite{liu2019task}, and autonomous navigation \cite{chu2012local}. When only two agents interact, there is a single collision-avoidance constraint, which can be easily managed using a safety filter. However, with multiple nearby agents, the resolution of a constraint between two agents can conflict with a constraint involving a third agent. These conflicts may result in suboptimal task performance, such as creating a severe gridlock that prevents agents from taking actions to achieve their tasks. More crucially, the inability to simultaneously satisfy all constraints can result in an agent taking an action that makes collision inevitable. 
In particular, this issue poses a significant safety risk in high-density scenarios like air taxi operations for Advanced Air Mobility (AAM) \cite{faa}.

Prior works have addressed safe multi-robot coordination problems by using model-based control methods like control barrier functions (CBFs) \cite{wang2017safety} and reachability analysis \cite{wang2020reachability}. %
Although CBFs and reachability provide a framework for safety assurance, they generally offer rigorous guarantees only when a single safety constraint is considered. The fundamental challenge in extending these methods to the multi-agent case is that the intersection of the safe sets corresponding to individual constraints (each derived from a pair of agents) does not necessarily represent the true safe set when all constraints are considered together (see Figure \ref{fig:running-example} for an example). In the Hamilton-Jacobi (HJ) reachability literature, the gap between these two regions is referred to as the ``leaky corner" \cite{mitchell2011scalable}. Agents that enter a leaky corner can no longer satisfy all safety constraints simultaneously and are inevitably forced to violate at least one.
Unfortunately, identifying leaky corners without recomputing the reachability analysis from scratch while incorporating all constraints remains an open problem \cite{lee2019, jiang2024guaranteed}. Performing reachability analysis or designing CBFs for all possible combinatorial interaction scenarios is computationally intractable. In summary, the fundamental challenge in achieving scalability with such control-theoretic methods in multi-agent settings lies in handling conflicting constraints.

In this work, we combine the control barrier-value function (CBVF) \cite{choi2021robust}, which is a CBF design method based on Hamilton-Jacobi reachability, with multi-agent reinforcement learning (MARL) into a layered safety architecture. \jcnote{This integration is driven by the essential role MARL can play in learning to strategically optimize task performance in multi-agent scenarios while proactively navigating potential conflicting constraints, which helps achieve safer and more effective behaviors.} 
As a result, our approach enhances both safety and performance to a level that neither safe control methods nor MARL alone could achieve.

\subsection{Contributions}
\begin{enumerate}
\item \textit{Architecture:} We propose a layered architecture that combines safety-informed MARL-based policy and CBVF-based safety filtering mechanism (Figure \ref{fig:summary_fig}), which can significantly mitigate the issues arising from conflicting constraints, such as inefficiency due to gridlock and the leaky corner problem.
\item \textit{Training method:} We propose a method to incorporate a CBVF-based safety filter into the training of MARL, considering two key aspects. First, a main challenge in this safety-constrained training is that the conservativeness introduced by safety filtering can hinder the exploration necessary for MARL to learn an effective policy. To address this, we introduce curriculum learning into the application of the safety filter, carefully balancing safety and exploration. Second, based on reachability analysis, we derive a conservative estimate of the safe region that is free from the issue of conflicting constraints (represented by the range $\dengagement$ in Figure \ref{fig:summary_fig}). Based on this estimate, the MARL policy is informed to minimize entry into this region, thereby avoiding potential conflicting constraints. Crucially, unlike many existing methods \cite{bejarano2024safety, zhang2023neural}, our proposed training approach does not impose safety through penalty terms directly penalizing the safety violation. Instead, MARL learns to enhance safety by making strategic decisions that mitigate conflicting constraints. This indirect approach significantly reduces unnecessary conservativeness, a common side effect of safe reinforcement learning-based methods.
\item \textit{Experimental validation}: We conduct hardware experiments using Crazyflie drones and perform simulations of high-density AAM scenarios to validate our hypothesis.
\end{enumerate}

The remainder of this paper is organized as follows. Section \ref{sec:related_works} provides background on safety for multi-agent coordination. Section \ref{sec:problem_statement} describes the system, environment, and problem statement. Section \ref{sec:safety-analysis} presents the safety analysis of multi-agent problems under collision avoidance constraints. Sections \ref{sec:methodology} and \ref{sec:results} present our proposed \algoname{} approach, the experiments performed, and the results obtained. We discuss some limitations of our approach in Section \ref{sec:limitations}. Finally, we conclude and propose future work in Section \ref{sec:conclusion}.
\begin{figure}
    \centering
    \includegraphics[width=\linewidth]{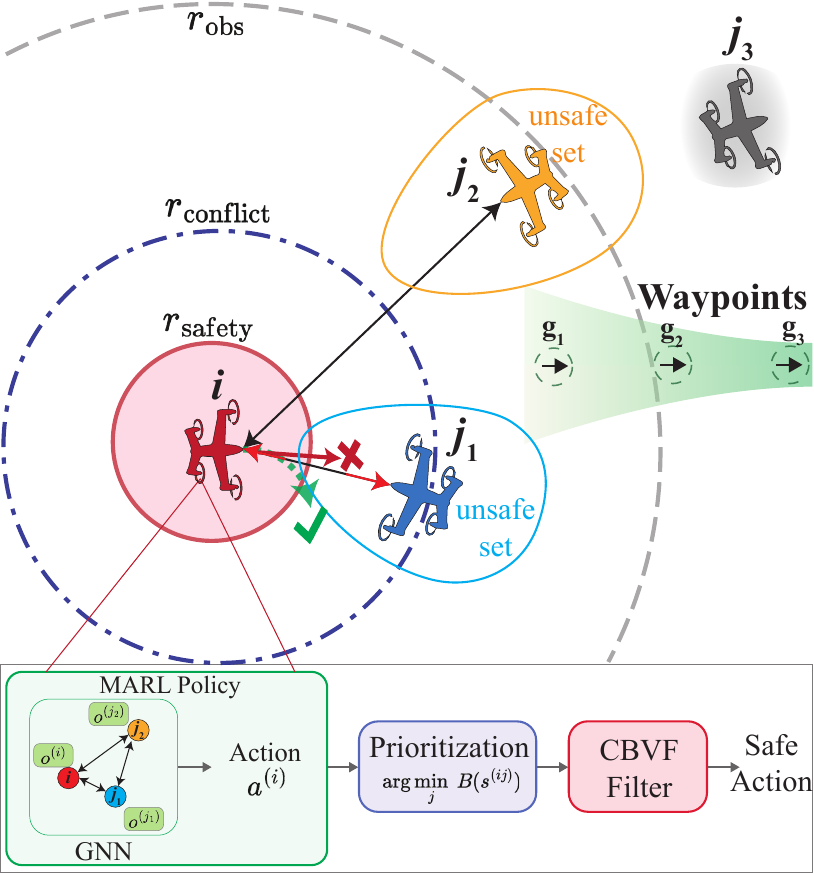}
    \caption{The figure shows our approach using an example scenario of four agents. Agent $i$ must reach the waypoints shown on the right. Our \algoname{} framework consists of three key components, and we describe it as applied through agent $i$: 1) The MARL policy generates an action based on the observation within the range $\dobservation$ while aiming to reduce the likelihood of entering other agents' potential conflict range $\dengagement$. 2) The prioritization module identifies the most critical neighboring agent in a potential collision scenario by evaluating the CBVF. In this example, agent $j_1$ is within the potential conflict region and forms a \textit{potential collision pair}. 3) The CBVF safety filter adjusts the action to ensure safe navigation.}
    \label{fig:summary_fig}
\end{figure}

\section{Related work}
\label{sec:related_works}
\subsection{Core related works---safety for multi-agent problems}
\textit{Classic Control barrier function (CBF)-based approaches.}
CBFs \cite{ames2016control} are used to design safe controllers via the principle of set invariance, and their application to safe multi-agent coordination has been explored in \cite{wang2017safety, jankovic2023multiagent, goarin2024decentralized}. A primary challenge in employing CBFs lies in constructing a valid CBF, which often requires system-specific, handcrafted designs \cite{wang2017safety}. In \cite{goarin2024decentralized, jankovic2023multiagent}, a more generic design principle based on exponential CBFs \cite{nguyen2016exponential} is employed; however, this approach does not address control input bounds. Another common limitation of existing methods is the treatment of multiple, potentially conflicting CBF constraints, which can lead to infeasibility. To address this, we adopt the CBVF-based framework to construct valid CBFs and handle multiple safety constraints in multi-agent coordination via a layered safety architecture.

\textit{Neural CBF-based approaches.}
While learning-based methods \cite{qin2021learning, zhang2023neural, zhang2025gcbf, zhang2025discrete} are proposed to design approximate CBFs, they lack deterministic safety guarantees due to the nonconvexity of the learning problems. 
Graphical CBF (GCBF) in \cite{zhang2023neural, zhang2025gcbf} offers a CBF based on local observations under multi-agent interaction, but how it learns to handle multiple constraints is not explicitly examined. Discrete Graphical Proximal Policy Optimization (DG-PPO) \cite{zhang2025discrete} proposes a model-free approach to learning decentralized CBFs and a safe control policy optimizing the task objective. Unlike DG-PPO, our approach leverages model-based information to compute Control Barrier Value Functions (CBVFs) \cite{choi2021robust} for pairwise collision avoidance, thereby ensuring deterministic safety guarantees. Finally, the aforementioned methods focus on learning safety certificates and policies for uncertain dynamical systems, often with high-dimensional system states. In contrast, our work specifically addresses the challenge of conflicting constraints in multi-agent interactions---a critical issue that persists even when each agent’s dynamics can be effectively represented by simple, low-dimensional models.

\textit{Reachability for multi-agent interaction.} Classical HJ reachability analysis computes the set of states that are guaranteed to be safe by computing the optimal control value function with dynamic programming. Prior works have investigated the reachability analysis for the special case of three-agent interaction \cite{ClaireHJ16} and using the value function to guide the planner, and its optimal control law is used in the tracking controller for safe multi-agent interaction \cite{wang2020reachability}.
Due to the curse of dimensionality in dynamic programming \cite{bellman1957dynamic}, the applicability of HJ reachability to high-dimensional systems is inherently limited. Recent work leveraged deep learning techniques to learn high-dimensional reachability \cite{bansal2021deepreach, hsu2021safety, hsu2023isaacs, zhu2024safe,li2025certifiable}, demonstrating their use in multi-agent collision avoidance scenarios. However, the learned solution does not generalize to new scenarios involving different agents. Additionally, the question of how to certify the safety of the learned safe set is still an open research topic \cite{yang2024scalable, lin2024verification,li2025certifiable}. %
Finally, alternative methods for solving reachable sets with over-approximation have been used in the context of multi-agent problems and air traffic management \cite{taye2022reachability, bertram2020distributed}.

\textit{Safe multi-agent reinforcement learning (MARL).} A common approach to safe MARL is through constrained Markov decision processes (CMDPs) \cite{altman2021constrained}. In theory, CMDPs have no duality gap under certain assumptions \cite{paternain2019constrained}, but in practice, training with PPO-Lagrangian \cite{ray2019benchmarking}, and its multi-agent variant \cite{gu2021multi} often suffers from instability due to suboptimal policies and inaccurate Lagrange multipliers. Another approach is shielded MARL, which uses safety filters \cite{hsu2023safety, bejarano2024safety} to enforce safety during training and deployment. Originally introduced for single-agent RL \cite{alshiekh2018safe}, this method has been extended to multi-agent settings \cite{elsayed2021safe}. %
However, designing a shielding policy remains challenging due to the curse of dimensionality.

\subsection{Other related works}
\textit{Multi-agent reinforcement learning.}
Multi-agent extensions of single-agent RL algorithms, such as PPO \cite{schulman2017proximal} and DDPG \cite{lillicrap2015continuous}, include MA-PPO \cite{yu2022surprising} and MA-DDPG \cite{lowe2017multi}, both of which assume full observability. However, in many real-world applications, such as autonomous driving, agents have to make decisions based on their local information and coordinate effectively with other agents. A key challenge in MARL is the decentralized decision-making under partial observation. InforMARL \cite{informarl_icml} leverages graphical neural networks for information sharing to develop an efficient, coordinated learning framework for acquiring a high-performance MARL policy. Our approach builds on InforMARL to allow agents to make decentralized decisions based on their local observations. %

\textit{Control and game-theoretic methods.} In collaborative multi-agent settings, Model Predictive Control (MPC) has been used to ensure safe control \cite{slegers2006nonlinear, el2015distributed, eren2017model, zhou2017fast, goarin2024decentralized} and its integration with MARL is explored in \cite{dawood2024safemultiagentreinforcementlearning}. However, the complexity of the constrained optimization involved in MPC often limits its real-time execution in complex systems. When agents pursue distinct objectives, the problem becomes a non-cooperative game, with various solution approaches proposed in prior work \cite{bacsar1998dynamic,evans2016airline,mylvaganam2017differential,fridovich2020efficient, bhatt2023efficient, FisacWhoPlays}. %
However, treating the safety constraints in the game-theoretic solutions remains an open research challenge \cite{laine2023computation,li2024computation}.

\textit{Collision avoidance \& conflict resolution for air traffic control.} With the growing interest in AAM applications such as drone deliveries and air taxi services, developing a scalable low-altitude air traffic management system that is automated or semi-automated has become an urgent need. Compared to current aviation, AAM operations are expected to be large-scale, ad hoc, on-demand, and dynamic. These characteristics motivate the development of a new air traffic management (ATM) framework that can achieve scalable, efficient, and collision-free operations \cite{faa, Kopardekar2016UnmannedAS}. 

Existing work on collision avoidance and conflict resolution for ATM is categorized into \textit{strategic deconfliction}, which focuses on preemptive deconfliction, and \textit{tactical deconfliction}, which focuses on imminent proactive collision avoidance. A substantial body of work leverages control theory to design methods for strategic deconfliction. An early work proposed a flight mode switching framework derived from a hybrid automaton and reachability-based analysis \cite{tomlin1998}. As this method suffers from the computational complexity of HJ reachability, \cite{ClaireHJ16} uses a mixed integer program to assign avoidance responsibilities and resolve conflicts cooperatively. The work in \cite{chen2017reachability} alternatively organizes vehicles into platoons on structured air highways, treating each platoon as a coordinated entity. While these methods provide strong safety guarantees, they rely on a predefined set of coordination rules for those guarantees to hold. Additionally, the approach in \cite{subliminal} uses preemptive strategic speed adjustments to prevent perceived conflicts without requiring controller intervention. Finally, a negotiation-based framework is introduced in \cite{Min2022negotiation} for collision-free strategic planning.

In parallel, the aviation community employs tactical collision avoidance modules as the final layer for safety. The Traffic Alert and Collision Avoidance System (TCAS) is an onboard system developed in the 1980s for conventional airliners, designed to detect and prevent collisions through vertical separation \cite{kuchar2004safety}. A method for tactical collision avoidance through horizontal resolution is also proposed in \cite{erzberger2010algorithm}. The successor of TCAS, the Airborne Collision Avoidance System (ACAS) X, integrates predictive modeling with real-time sensor inputs \cite{holland2013optimizing} using a partially observable Markov decision process framework. These existing methods crucially rely on the assumption that no more than two vehicles are involved in a single conflict resolution. This assumption is typically upheld by the upstream strategic deconfliction decisions.

Various methods in both strategic and tactical deconfliction are integrated further into layered, hierarchical decision-making architectures, enhancing the safety of ATM \cite{TOMLIN1996, tang2008tactical, erzberger2012automated}. Our work is inspired by these layered approaches in the aviation community; however, the separations between layers underlying the existing approaches do not directly apply to high-volume AAM scenarios. As such, we have to consider how to achieve safe collision avoidance in instances of simultaneous multi-vehicle engagement.

Finally, MARL-based methods have also been explored in air traffic control to ensure tactical deconfliction through preconditioned strategic planning \cite{Chen&PengWei}, demonstrating improved safety and efficiency over rule-based methods. However, the available actions of each agent in this work are limited to the adjustment of speed or position.

\section{Problem Formulation}
\label{sec:problem_statement}

In this section, we define the system, environment, each agent's dynamics, and their safety requirement, and the problem statement.
\subsubsection{Preliminaries}
We formulate our multi-agent system as a Decentralized Partially Observable Markov Decision Process (Dec-POMDP) defined by the tuple $\langle N, S, O, \mathcal{A}, P, R,\gamma \rangle$, where:
\begin{itemize}
    \item $N$ is the number of agents
    \item $s^{(i)}\in \mathbb{R}^D$ is the state of each agent with $D$ as the state dimension, including their position variables, %
    \item $s\in S = \mathbb{R}^{N\times D}$ is the environment state, which is the concatenation of each agent's states and the state space of the environment, respectively,
    \item $o^{(i)}=O(s^{(i)})\in \mathbb{R}^{d}$ is the observation of agent $i$,
    \item $\aagent \in \aset$ is the action space for agent $i$. $\aseq$ denotes the sequence of actions for timesteps $k = 0, 1, \cdots$, %
    \item $P(s'|s, a)$ is the transition probability from $s$ to $s'$ given the joint action $a$, the concatenation of each agent's actions,
    \item $R(o^{(i)}, a^{(i)})$ is the common reward function of all agents,
    \item $\gamma \in [0,1)$ is the discount factor.
\end{itemize}
The objective is to find a policy $\Pi= \left(\pi^{(1)}, \cdots, \pi^{(N)}\right)$, where $\pi^{(i)}\left(a^{(i)}|o^{(i)}\right)$ is agent $i$'s policy that selects an action based on its observation.

\subsubsection{Agent's dynamics \& safety constraint}

We consider each agent's dynamics as a sampled data system, meaning that their underlying physical dynamics evolve continuously in time, but their actions are updated at discrete timesteps. Their continuous dynamics are given as
\begin{equation}
\label{eq:agent-dynamics}
\dot{s}^{(i)}(t) = f^{(i)} (\sagent(t), \aagent(t)), \quad \sagent(0) = \sagent_0,
\end{equation}
and their action is updated at every sampling time $\samplingtime$---i.e. the action sequence $\aseq$ maps to the signal in time given as $\aagent(t) \equiv \aagent_\timestep$ for $t \in [\timestep \Delta t, (\timestep+1) \Delta t )$. Their discrete-time state is given as $\sagent_\timestep = \sagent(\timestep \Delta t)$.

The primary safety constraint we consider in this work is the collision avoidance between agents. For all time $t \ge 0$, agents must satisfy
\begin{equation}
\label{eq:safety}
    \distance(\sagent(t), \sagentj(t)) \ge \dsafety, \quad \text{for}\;\; \forall i \neq j,
\end{equation}
where $\dsafety$ is the safety distance.

In the subsequent safety analysis, we consider the relative dynamics between a pair of agents, $(i, j)$. We define the relative state between the agents, which can be given as $\srelative := \mathrm{rel}(\sagent, \sagentj)$, where $\mathrm{rel}$ is a mapping from two agents' states to the relative state. We assume that relative position variables are part of $\srelative$; thus, $\distance$ can be defined based on $\srelative$. The dynamics of the relative state are described by 
\begin{equation}
\label{eq:relative}
\dot{s}^{(ij)}(t) = f^{(ij)} \big( \srelative(t), \aagent(t), \aagentj(t)\big),    
\end{equation}
which is derived from \eqref{eq:agent-dynamics}.

\subsubsection{Observations}

For each agent to learn an effective policy for performance and safety, the observations $\oagent$ need to contain adequate information. We make the following assumptions that are generic for many multi-agent robot tasks.
\begin{itemize}
    \item Each agent $i$'s observation $\oagent$ consists of its local observations of other agents and entities relevant to their task goals (e.g., goal location) within their observation range defined as $\dobservation$ and any additional information needed for its task. Thus, the reward given to agent $i$ at each timestep, $R(\oagent, \aagent)$, is defined based on its local observation and action.
    \item We define $I(i):\mathbb{N}\!\rightarrow\!2^{\mathbb{N}}$ as the index set of the agents within the observation range of agent $i$. We assume that $o^{(i)}$ contains information that can be used to reconstruct $\srelative$ from $o^{(i)}$ for all $j\!\in\!I(i)$. Thus, for agent $i$, with its observation, it is feasible to execute a feedback policy on $\srelative$ if agent $j$ is within its observation range. This assumption will be used in the design of our safety framework.
\end{itemize}

\subsubsection{Problem statement}

To sum up, the decentralized multi-agent coordination problem, subjected to the collision avoidance constraint we consider in this work, can be described as
\begin{equation}
    \begin{aligned}
        \max_{\pi^{(i)}} &  \; \mathbb{E} \left[ \sum_{k=0}^\infty R(\oagent_k,\aagent_k) \right] \\
        \textrm{s.t. }& s_{k+1} \sim P(s \;| \; s_k,a_k)\\
        & a_k^{(i)}\sim \pi^{(i)} (a^{(i)}\;|\; o_k^{(i)})  \\ %
        \distance(\sagent(t),  \sagentj(t)&) \ge \dsafety, \quad \text{for}\;\; \forall i \neq j, \forall t \ge 0,
    \end{aligned}
\end{equation}
where each agent's action $a_k^{(i)}$ is determined by their policy $\pi^{(i)}$, based on their local observations. The agent learns to maximize its objective subject to its collision avoidance constraint.
\section{Safety Analysis}
\label{sec:safety-analysis}
In this section, we present the safety analysis of the multi-agent problem under collision avoidance constraints. Specifically, we first derive the safety analysis for a pair of agents and then investigate how it applies to the multi-agent scenario.

\subsection{Collision avoidance for a pair of agents}

To ensure $\distance(\srelative(t)) \ge \dsafety$ for all $t \ge 0$, we consider the following cost function, which captures the closest relative distance along the trajectory:
\begin{equation}
\label{eq:reachability-cost}
    J(\srelative_0, \aseq, \aseqj) := \min_{t \in [0, \infty)} \distance \big( \srelative(t) \big).
\end{equation}
If $J(\srelative_0, \aseq, \aseqj) \ge \dsafety$, the agents $i$ and $j$ are rendered safe (collision-free) by their actions.

\subsubsection{Reachability analysis for computing the maximal safe set}\label{subsubsec:reachability}

The agent pair prioritizing safety would want to maximize \eqref{eq:reachability-cost} to move away from each other. From this intuition, we can consider the following optimal control problem
\begin{equation}
\label{eq:reachability-value}
V(\srelative_0) := \max_{\aseq, \aseqj} J(\srelative_0, \aseq, \aseqj).
\end{equation}
Solving $V$ is a specific type of reachability problem called the minimal Backward Reachable Tube (BRT) problem \cite{wabersich2023}. To see this, consider $\unsafeset = \{\srelative \;| \;\distance(\srelative) < \dsafety \}$, the near-collision region, as the target set. The minimal BRT of $\unsafeset$ is defined as
\begin{equation}
\brt(\unsafeset) := \{\srelative_0 \;| \; \forall \aseq, \aseqj, \exists t \ge 0 \; \text{s.t.}\; \srelative(t) \in \unsafeset\},
\end{equation}
which encapsulates a region from which no action sequence can prevent the relative state from entering the near-collision region $\unsafeset$. Using the definition in \eqref{eq:reachability-value}, we can express $\brt(\unsafeset)$ as
$\{ \srelative_0 \; | \; V(\srelative_0) < \dsafety \}$. 

The complement of $\brt(\unsafeset)$ becomes the \textit{maximal} safe set from which the agent pair can avoid collisions since it encompasses all the states from which there exist action sequences $\aseq$ and $\aseqj$ that can avoid collision. This maximal safe set is denoted as
\begin{equation}
    \label{eq:pair-safe-set}
\safesetpair := \{\srelative_0 \;| \; \exists \aseq, \aseqj,\; \text{s.t.}\; \forall t \ge 0,  \srelative(t) \notin \unsafeset\},
\end{equation}
and satisfies
\begin{equation}
\safesetpair = \big(\brt(\unsafeset)\big)^c = \{ \srelative_0 \; | \; V(\srelative_0) \ge \dsafety \}.
\end{equation}

\noindent We use the open-source library in \cite{hj_reachability} to compute \eqref{eq:reachability-value} and $\safesetpair$, which computes the Hamilton-Jacobi (HJ) partial differential equation (PDE) associated with the BRT problem \cite{bansal2017hj}. 

\begin{figure*}[t!]
    \centering
    \includegraphics[width=\textwidth]{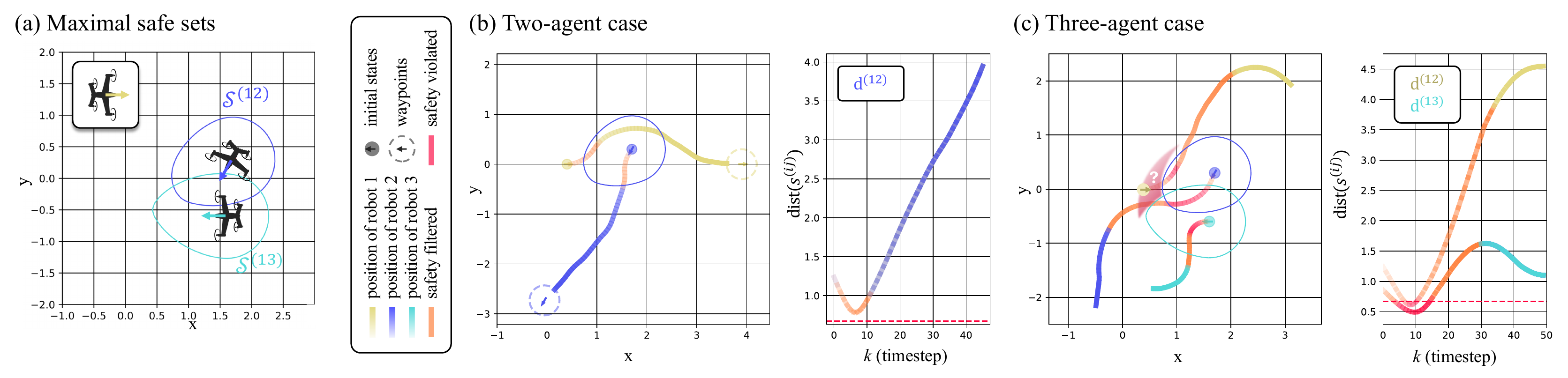}
    \vspace{-2em}
    \caption{Running example illustrating the CBVF-based safe sets, safety filtering, and the leaky corner issue. (a) Visualization of the ego agent ($s^{(1)} = [0.4\si{km}, 0\si{km}, 0\degree, 110$ $\si{kt}]$)'s maximal safe sets (exterior of the level sets) against two agents, $s^{(2)}\!=\![1.7 \si{km}, 0.3\si{km}, -120\degree, 110\si{kt}]$ and $s^{(3)} \!=\![1.7 \si{km}, -0.6\si{km}, -180\degree, 60\si{kt}]$. (b) In the two-agent case, each agent executing their CBVF safety filters \eqref{eq:cbvf-qp-single} successfully prevents collision. (c) In the three-agent case, although agent 1 started inside the intersection of $\mathcal{S}^{(12)}$ and $\mathcal{S}^{(13)}$, it is not able to prevent safety violation. This is because the initial state of robot 1 is in the leaky corner.}
    \vspace{-1em}
    \label{fig:running-example}
\end{figure*}

\begin{example}
We consider the reduced-order dynamics of an autonomous air taxi given as 
\begin{equation}
\begin{aligned}
 \dot{x}=v \cos \theta,\;\; \dot{y}=v \sin \theta, \;\; \dot{\theta}= \omega, \;\;\dot{v} = \mathrm{a},
\end{aligned}
\label{eq:airtaxi}
\end{equation}
where the robot state consists of $\sagent = [x;y;\theta;v]$, representing the positions, heading, and speed. The allowable actions are  $\aagent = [\omega, \mathrm{a}]$, corresponding to the angular rate and the longitudinal acceleration, respectively. The speed is limited to the range of $[v_{\min}, v_{\max}]$. The action space is defined as $\aset = [-\omega_{\max},\omega_{\max}]\times [\mathrm{a}_{\min}, \mathrm{a}_{\max}]$. The parameters we use are defined in Table \ref{tab:vehicle_dynamics} and are explained in more detail in Section \ref{subsec:bay-area-sim}.

The relative state $\srelative = [x^{(ij)};y^{(ij)};\theta^{(ij)};v^{(i)};v^{(j)}]$ includes the relative position and heading of agent $j$ from agent $i$'s perspective, where $x$-axis is in the direction of the agent $i$'s heading. The relationship between $\srelative$ and ($\sagent$, $\sagentj$) and the relative state dynamics are given in Appendix \ref{appendix-airtaxi-dynamics}.

The computation of $V$ was completed within an hour using an Nvidia RTX A4500 GPU. \footnote{The computation time is not a critical concern in our setting, as the value function is computed offline rather than during real-time deployment.} 
The computed maximal safe set $\safesetpair$, defined in the relative state space of $\srelative$, can be projected to the position space of the ego agent (agent $i$), which incorporates all safe positions of the agent that can ensure collision avoidance, given its heading, speed, and the opponent agent (agent $j$)'s state. Examples of $\safesetpair$ are visualized in Fig. \ref{fig:running-example} (a) with respect to two different opponent states. 
\end{example}

\subsubsection{Control barrier-value function-based safety filtering}

Next, we investigate how the computed value function $V$ can be used to constrain the relative state $\srelative$ to stay within the safe set $\safesetpair$. Since each agent makes their primary decision based on their MARL policy in our framework, we consider how to filter the MARL action if it is potentially unsafe. 

To achieve this safety filtering mechanism, we consider the barrier constraint-based mechanism of the CBFs. For a generic state variable $s$ and its dynamics $\dot{s} = f(s, a)$, if a function $B(s)$ satisfies the barrier constraint given by
\begin{equation}
    \nabla B(s) \cdot f(s, a) + \gamma B(s) \ge 0,
\label{eq:barrier-constraint}
\end{equation}
for every state inside the zero-superlevel set of $B$, i.e. $\forall s \in \{ s \;|\; B(s) \ge 0 \}$, and for some constant $\gamma > 0$, we can guarantee that $B(s(t)) \ge 0$, for all $t \ge 0$ \cite{ames2016control}. Thus, the state can be maintained to stay within $\{ s\; | \; B(s) \ge 0 \}$.

If the computed reachability value function $V$ in \eqref{eq:reachability-value} is almost-everywhere differentiable, we can construct a CBF by taking $B(\srelative) = V(\srelative) - \dsafety$. 
This choice of $B$ satisfies the barrier constraint \eqref{eq:barrier-constraint} almost everywhere, and results in the maximal safe set to be represented as the CBF zero-superlevel set, $\safesetpair = \{\srelative \; | \; B(\srelative) \ge 0 \}$. Such usage of the reachability value function as the CBF is referred to as the Control Barrier-Value Function (CBVF) \cite{choi2021robust}.

\begin{remark}
In \cite{zhang2025discrete}, $V$ is denoted as a constraint-value function and is used to learn Graphical CBF (GCBF) for uncertain dynamics. In this work, we consider its exact computation for a hard safety guarantee. However, our approach can be combined with the learning methods proposed in \cite{zhang2025discrete} or other learning-enabled approaches \cite{fisac2019bridging, chilakamarri2024reachability} to be extended to agents subjected to uncertain dynamics.
\end{remark}

\begin{remark}
For certain types of dynamics, the value function can be discontinuous without introducing a discount factor in time to the cost function \eqref{eq:reachability-cost} \cite{choi2023forward}. %
\end{remark}

Finally, the CBVF-based safety filtering can be implemented in a decentralized manner, with each agent executing its own safety filter. Here, we assume that the agents are \textit{cooperative for safety}, meaning that although their unfiltered actions can be selfish, their final filtered actions are coordinated to avoid collision with each other. To achieve this coordination, agent $i$ and agent $j$ can individually solve the identical optimization program defined as 
\vspace{0.5em}
\hrule
\vspace{0.5em}
\noindent CBVF Safety Filter (cooperative case):
\begin{align}
(\asafe &, \;\asafej) = \arg\!\!\!\!\!\!\!\min_{(\aagent, \aagentj) \in \aset} ||\aagent\!-\!\amarl||^2 + ||\aagentj\!-\!\amarlj||^2 \nonumber \\
\textrm{s.t.} \;\; & \nabla B(\srelative)\!\cdot\!  f^{(ij)} \big( \srelative, \aagent, \aagentj\big)\! + \!\gamma B(\srelative) \ge 0,
\label{eq:cbvf-qp-single}
\end{align}
\hrule
\vspace{0.5em}
\noindent and then execute their own action. If the dynamics $f^{(ij)}$ are affine in actions, the optimization becomes a quadratic program \cite{ames2016control, choi2021robust}.

If the opponent agent is non-cooperative for safety, agent $i$ can solve for its own safe action considering the worst-case possible action of the opponent:
\vspace{0.5em}
\hrule
\vspace{0.5em}
\noindent CBVF Safety Filter (non-cooperative case):
\begin{align}
\asafe & = \arg\!\!\min_{\aagent \in \aset} ||\aagent\!-\!\amarl||^2 \nonumber \\
\textrm{s.t.} \;\; & \min_{\aagentj \in \aset} \nabla B(\srelative)\!\cdot\!  f^{(ij)} \big( \srelative, \aagent, \aagentj\big)\! + \!\gamma B(\srelative) \ge 0, \nonumber
\end{align}
\hrule
\vspace{0.5em}
\noindent where now $B$ has to be constructed based on a value function for a differential game, which considers the opponent's worst-case actions \cite{mitchell2005}, given as 
\begin{equation}
\label{eq:reachability-value-worst-case}
V_{\textrm{worst}}(\srelative_0) := [\;\min \hspace{-2em}\max_{\{\aagentj_k, \aagent_k\}_{k\ge0}\;\;\;\;\;\;\;\;} \hspace{-2em}] \; J(\srelative_0, \aseq, \aseqj).
\end{equation}
Here, $[\min \max]$ denotes the alternating operation $\min$ (over $\aagentj_k$) and $\max$ (over $\aagent_k$). The computation of this worst-case value function can be done similarly to the computation of $V$ by solving the min-max HJ PDE \cite{bansal2017hj}.

\begin{example}
In the two-agent case, in Figure \ref{fig:running-example} (b), the initial relative state between agents 1 and 2 is set near the boundary of the maximal safe set $\mathcal{S}^{(12)}$. By each agent applying the CBVF safety filter, both agents reach their goals safely under the safety-filtered MARL actions.
\end{example} 

\subsection{Analysis of the multi-agent case}

We begin the analysis of this section by continuing with the running example of the multi-agent case:

\begin{example}
In Figure \ref{fig:running-example} (c), we now consider the case where a third agent is introduced. The relative states still remain within the pairwise maximal safe sets $\mathcal{S}^{(12)}$, $\mathcal{S}^{(13)}$, and $\mathcal{S}^{(23)}$. Despite all agents actively attempting to avoid collisions, their relative distances fall below $\dsafety$. As mentioned in the introduction, this demonstrates the issue of conflicting constraints, implying that although agent 1's initial state did not cross the boundaries of the individual safe sets, it may already be outside the true safe set when considering all interactions simultaneously. Computing this true safe set requires defining the relative dynamics of the three agents, which increases the system's dimensionality. While approximations of this set have been computed, such as in \cite{bansal2021deepreach}, the computation of this multiple-agent safe set is challenging.
\end{example}

As can be seen in the above running example, it is crucial to prevent the agents from falling into the region in which one safety constraint can potentially conflict with the other, i.e., the leaky corners. Although their exact computation is hard, we define the region that can tightly overapproximate this potential conflict region.

\begin{proposition}
\label{prop:conflict-free-region}
Define 
\begin{equation}
\resizebox{0.98\hsize}{!}{$\displaystyle
\begin{aligned}
& \hat{\mathcal{S}}^{(i)} := \Big\{\{\sagentj\}_{j \in I(i)} \;| \; \forall j \in I(i), V(\srelative) \ge \dsafety, \& \\
& \nexists j_1, j_2 \in I(i)\; \text{s.t.} \ V_{\textrm{worst}}(s^{(ij_1)})\!<\!\dsafety\;\&\; V_{\textrm{worst}}(s^{(ij_2)})\!<\!\dsafety\Big\}
\label{eq:def-conflict-free-region}
\end{aligned}
$}
\end{equation}
where $V$ is defined in \eqref{eq:reachability-value}, and $V_{\textrm{worst}}$ is defined in \eqref{eq:reachability-value-worst-case}. Note the difference between $V$ and $V_{\textrm{worst}}$. Then for any opponent agent states $\{\sagentj\}_{j \in I(i)}\in\hat{\mathcal{S}}^{(i)}$, there exists $\aseq$ and $\aseqj$ for all $j \in I(i)$, such that $\forall t \ge 0$,  $\srelative(t) \notin \unsafeset$ for all $j \in I(i)$. In other words, set $\hat{\mathcal{S}}^{(i)}$ can be maintained forward invariant.
\end{proposition}
\begin{proof} The second condition in \eqref{eq:def-conflict-free-region} allows at most one opponent agent to enter the area in which $V_{\textrm{worst}}(\srelative) < \dsafety$. We denote this agent as $j_{\textrm{near}}$. For $(i, j_{\textrm{near}})$, since $V(s^{(i j_{\textrm{near}})}) \ge \dsafety$ based on the first condition in \eqref{eq:def-conflict-free-region}, agent $i$ and agent $j_{\textrm{near}}$ are within their CBVF safe set $\mathcal{S}^{(ij_{\textrm{near}})}$ and can select their action sequences $\aseq$ and $\boldsymbol{a}^{(j_{\textrm{near}})}$, such that $s^{(i j_{\textrm{near}})}(t) \notin \mathcal{L}^{(i j_{\textrm{near}})}$ for all $t\ge0$.

Next, we consider all other agents $j \in I(i) \setminus \{j_{\textrm{near}}\}$. Based on \cite[Proposition 4]{fisac2019tac}, for any Lipschitz continuous $V_{\textrm{worst}}$, its level set is a robust control invariant set. Thus, for all $j \in I(i) \setminus \{j_{\textrm{near}}\}$, there exists $\aseqj$ that results in $V_{\textrm{worst}}(\srelative (t)) \ge \dsafety$ for all $t \ge 0$, regardless of $\aseq$, ensuring $\sagentj(t) \in \hat{\mathcal{S}}^{(i)}$. 
\end{proof}

Intuitively, the set $\hat{\mathcal{S}}^{(i)}$ prevents conflict of constraints by allowing only one agent to coordinate for collision avoidance with the ego agent and by prohibiting the other agents from entering the worst-case safe set. These other agents are able to stay away from the pair $(i, j_{\textrm{near}})$ due to the robust invariance property of the level set of $V_{\textrm{worst}}$.

\textit{Practical implementation: } In practice, enforcing each agent to stay within $\hat{\mathcal{S}}^{(i)}$ can be computationally expensive since we have to evaluate $V$ and $V_{\textrm{worst}}$ for all pairs of interaction. In the next section, we use this analysis to inform MARL to implicitly learn not to enter this region. For this, we define the \textit{potential conflict range} as below:
\begin{align}
\label{eq:conflict_distance}
    \dengagement & := \min_{r} \; r \\
\text{s.t.}  & \quad V_{\textrm{worst}}(\srelative) \ge \dsafety \;\;\forall \srelative \;\text{s.t.}\; \distance(\srelative) \ge r. \nonumber
\end{align} 
Then, the set of opponent agent states is defined as
\begin{equation}
\resizebox{0.98\hsize}{!}{$\displaystyle
\begin{aligned}
\label{eq:def-conflict-free-region-approx}
& \tilde{\mathcal{S}}^{(i)} := \Big\{\{\sagentj\}_{j \in I(i)} \;| \; \forall j \in I(i), V(\srelative) \ge \dsafety, \& \\
& \hspace{-0.5em}\nexists j_1, j_2 \in I(i) \; \text{s.t.} \;\distance(s^{(i j_1)})\!<\!\dengagement\;\&\; \distance(s^{(i j_2)})\!<\!\dengagement\Big\}
\end{aligned}
$}
\end{equation}
This is an underapproximation of the true conflict-free set $\hat{\mathcal{S}}^{(i)}$ by definition \eqref{eq:conflict_distance}. To ensure safety, we want to restrict the number of opponent agents entering this region to be at most one. %

Our analysis requires that the observation range be larger than the potential conflict range, $\dobservation > \dengagement$. Rather than a restriction, this serves as a design guideline for the observation stack of the robot for safe multi-robot coordination. As shown in Figure~\ref{fig:summary_fig}, the concept of the potential conflict range divides a robot’s proximity into three layers: (1) the range $\distance(\srelative) < \dsafety$, where collision is imminent; (2) the range $\dsafety < \distance(\srelative) < \dengagement$, where engaging with multiple vehicles may introduce safety risks; and (3) the region $\dengagement < \distance(\srelative)$, where the maneuvers of other agents pose minimal safety concerns. A similar three-layer structure was proposed and manually designed in \cite{ghosh2000maneuver}. However, our approach provides a theoretical foundation for defining these boundaries based on reachability analysis.

\begin{remark}
\label{remark:theory-limitation}(Limitation) The set $\hat{\mathcal{S}}^{(i)}$ is a conflict-free safe set only from agent $i$'s perspective. In other words, it does not guarantee that the collision-avoidance maneuvers of other agents $j \in I(i)$ will not interfere with one another. Addressing this issue requires analyzing the combinatorial number of possible interactions, which remains an open problem. In our work, we address this challenge by training the MARL policy to learn strategies that mitigate these conflicts.
\end{remark}

\section{Multi-agent Reinforcement Learning with Layered Safety}
\label{sec:methodology}
\subsection{Extending InforMARL for improved decentralized decisions}
Our work builds upon the InforMARL architecture \cite{informarl_icml}, a MARL algorithm that solves the multi-agent navigation problem by using a graph representation of the environment, enabling local information-sharing across the edges of the graph. InforMARL uses graph neural networks (GNNs) to process neighborhood entity observations, allowing the framework to operate with any number of agents and provide scalability without changing the model architecture. Each agent has a set of neighboring agents within its observation range, $\dobservation$, and shares its relative position, speed, and goal information with these neighbors.
Agents are tasked to navigate to their respective goal positions.
When agents reach their respective goals, they get a goal reward ${\mathcal{R}}_\mathrm{goal}(\oagent_\timestep, a^{(i)}_\timestep)$.

\jcnote{The extensions we make to the baseline InforMARL to incorporate the layered safety framework and to make it more practical for multi-robot navigation tasks are as follows:}

1) Sequential goal point tracking: In the updated framework, the agents navigate to a sequence of waypoints, each specified by its position and the desired direction and speed, leading to the final goal (as shown by the green circles in Figure \ref{fig:summary_fig}). At each time step, an agent gets the following additional rewards, ${\mathcal{R}}_\mathrm{tracking}(\oagent_\timestep, a^{(i)}_\timestep)$ which are computed based on the heading and speed of the agent relative to the current target waypoint. \jcnote{The details of these terms are presented in Appendix \ref{appendix-goal-reward}.}

2) Model architecture enhancements:  To improve the algorithm and generalize it over diverse scenarios, we update the observations to incorporate rotation-invariant relative distances of the ego agent to goals and neighbors. 
Once an agent crosses a waypoint, we no longer consider the waypoint in its observation, and the agent moves to the subsequent waypoint.
We introduce dynamics-aware action spaces that are updated based on the dynamics model, angular rate, and longitudinal acceleration, as in the running example in Section \ref{sec:safety-analysis}, ensuring agents respect motion constraints specific to their dynamics.

3) Curriculum training: 
The training framework also incorporates curriculum learning where we progressively make the training environment harder  \cite{narvekar2020curriculum} for improving agents' performance, refining safety rewards, and updating the safety distance $\dsafety$ used in the safety filter. This is detailed in the subsequent sections. %

\subsection{Safety filter design for multiple agents}
For multiple agents, the CBVF $B(s^{(ij)})$ is evaluated for each agent $i$ and any neighboring agent $j$ within the observation range $\dobservation$. A smaller value of $B$ indicates that the near-collision is more imminent and safety is at greater risk. The neighbor agent with the minimum pairwise $B(s^{(ij)})$ is selected as the agent whose actions will be curtailed. We term the module that selects this prioritized constraint as the \textit{prioritization module}. 
If an agent pair $(i,j)$ has each other as the minimum pairwise $B(s^{(ij)})$, then we call them a \textit{potential collision pair}.
\subsection{Safety-informed training}
\label{subsec:safety-informed}

\subsubsection{Curriculum update}
We start the training routine without any safety filter or penalty applied for the first half of the training steps. This is done to optimize the task performance of MARL unconstrained by any safety parameters. Once training reaches half the number of total training steps, we activate the safety filter. Additionally, we introduce the following safety parameters, which are updated using the curriculum learning framework during training. First, the safety distance $\dsafety$ is initialized to zero during the start of model training, allowing agents to approach each other at close ranges. As the training progresses, we increase $\dsafety$ to the desired value. Similarly, we scale the conflict radius $\dengagement$ computed using Eq. \eqref{eq:conflict_distance} based on the value of the $\dsafety$. This setup allows agents to explore the environment early on in the training and prevents them from converging to overly conservative behavior.

\subsubsection{Safety-informed reward}

In addition to the heading, speed, and goal rewards, we introduce some additional penalties. When more than two agents are within the conflict radius $\dengagement$, we apply a potential conflict penalty 
\begin{equation}
\begin{aligned}
    \cost_\textrm{conflict} :=& \sum_{j \in \{ j | \distance(s^{(i j)})<\dengagement\}}\max\{0,r_{\textrm{conflict}} - \textrm{dist}(s^{(ij)}) \}  \\
        & \ \times \max\Big\{0, -\underbrace{\begin{bmatrix}
            x^{(ij)} & y^{(ij)}
        \end{bmatrix} \begin{bmatrix}
            v_x^{(ij)}  \\
            v_y^{(ij)}
        \end{bmatrix}}_{\textrm{relative distance change}}\Big\},
\end{aligned}
\end{equation}
which evaluates whether the agent $j$ is within the potential conflict range and is approaching towards agent $i$. Based on Proposition \ref{prop:conflict-free-region}, we do not apply the penalty when there is just one agent within $\dengagement$.

The penalty $\cost_\textrm{conflict}$ is carefully designed to mitigate the risks associated with potential conflicting constraints when multiple agents enter the range, while simultaneously minimizing the conservatism it may introduce. This penalty is an \textbf{indirect} safety penalty, as it is not incurred based on explicit safety violations but rather indirectly through the proximity of multiple agents.

The final reward structure is 
\begin{multline}
\mathcal{R}_\mathrm{total}(o^{(i)}_\timestep, a^{(i)}_\timestep) = {\mathcal{R}}_\mathrm{tracking}(o^{(i)}_\timestep, a^{(i)}_\timestep) \\+  \rho_{\mathrm{goal}} \mathcal{R}_\mathrm{goal}(o^{(i)}_\timestep, a^{(i)}_\timestep) - \rho_{\mathrm{conflict}} \cost_\mathrm{conflict}
\label{eq:totalreward}
\end{multline}
where $\rho_{\mathrm{goal}}$ is a binary indicator when the agent is at the goal, and $\rho_{\mathrm{conflict}}$ is a binary indicator when the number of other agents within the potential conflict region is more than one.

\begin{table}[t!]
\centering
\caption{Parameter Summary for Different Vehicle Dynamics}
\label{tab:vehicle_dynamics}
\begin{tabular}{|p{4cm}|p{2.4cm}|c|}
\hline
\textbf{Parameter} & \textbf{Air taxi (Sim)} & \textbf{Crazyflie} \\ \hline
\textbf{Groundspeed} & & \\ 
\quad \(v_{\min}\) & 60 knot (30 m/s) & -1.0 m/s\\ 
\quad \(v_{\max}\) & 175 knot (90 m/s) & 1.0 m/s \\ 
\quad \(v_{\text{nominal}}\) & 110 knot (57 m/s) & 0.5 m/s \\ \hline
\textbf{Acceleration} & & \\
\quad \(a_{\min}\) & -3.3 ft/s$^2$ (-1.0 m/s$^2$) & -0.5 m/s$^2$\\ 
\quad \(a_{\max}\) & 6.6 ft/s$^2$ (2.0 m/s$^2$) & 0.5 m/s$^2$\\ \hline
\textbf{Angular Rate ($\omega_{\max}$) (rad/s)} & 0.1 & - \\ \hline
\textbf{Sampling Rate (s)} & 1.0 & 0.1 \\ \hline
\textbf{Waypoint Thresholds ($\pm$)} & & \\ 
\quad Distance to Goal & 0.186 miles (0.3 km) & 0.2 m \\ 
\quad Heading & 45\degree  & 45\degree \\ 
\quad Speed & 38.9 knot (20 m/s) & 0.1 m/s \\ \hline
\textbf{Observation Range ($\dobservation$)} & 3.1 mi. (5.0 km) & 4.0  m\\ \hline
\textbf{Safety Distance ($\dsafety$)} & 500 - 2200 ft (0.152 - 0.671 km)  & 0.5  m\\ 
\textbf{Potential Conflict Range ($\dengagement$)} & 4600 ft (for $\dsafety$=2200  ft) & 1.0 m \\ \hline
\end{tabular}
\end{table}

\begin{figure}[h!]
    \centering
    \includegraphics[width=\columnwidth]{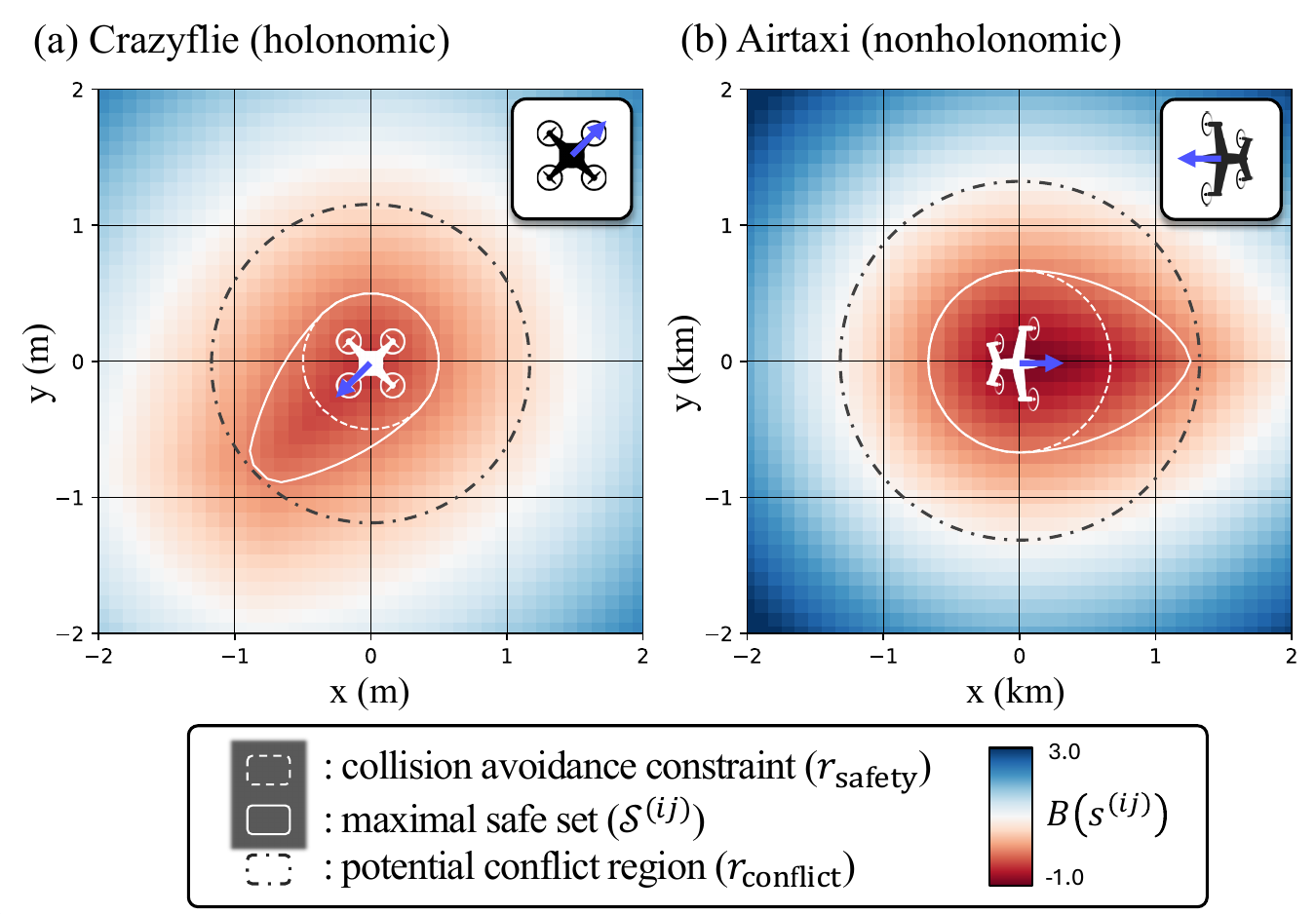}
    \vspace{-1.5em}
    \caption{Maximum safe sets (exterior of the white level sets), potential conflict region, and CBVF (colormap) for each vehicle dynamics, displayed in the relative position space when (a) relative velocity is $(v_x, v_y) = (1, 1)$ [m/s], (b) relative speed and heading is 220 knots and $180$\degree, respectively.}
    \label{fig:conflict-radius}
\end{figure}

\section{Results}
\label{sec:results}

The main robotic application we focus on is the safe autonomous navigation of aerial vehicles. We apply our framework to Crazyflie drones navigating through waypoints in both simulation and hardware experiments, as well as to the simulation of air taxi operations in realistic settings.

\subsection{Experiment Setup}

\textit{Considered dynamics.} We consider two types of dynamics, one for the quadrotors and the other for the air taxi vehicle in a wing-borne flight. The parameters for both dynamics are summarized in Table \ref{tab:vehicle_dynamics} and are set to match the industry standard \cite{joby_aviation, archer_aircraft, wisk_aircraft}. For instance, we use an angular rate bound of $0.1$ rad/s for the air taxi dynamics, as it results in the lateral acceleration $0.5g$ under the nominal speed, which amounts to the maximum tolerable lateral acceleration for passenger comfort in NASA market studies \cite{Shish}.

The quadrotor dynamics in the horizontal plane are represented as simple double integrators with
\begin{equation}
    \dot{x} = v_x,\;\; \dot{y} = v_y,\;\; \dot{v}_x = a_x,\;\; \dot{v}_y = a_y,
\end{equation}
where $\sagent = [x, y, v_x, v_y]$ and $\aagent=[a_x, a_y]$. The quadrotor runs the low-level onboard flight controller to track the commanded actions.

The air taxi dynamics in the horizontal plane are represented using the kinematic vehicle model in \eqref{eq:airtaxi} of the running example in Section \ref{sec:safety-analysis}. Three features of the air taxi dynamics considered in this work make its safety assurance more challenging and interesting. First, the vehicle cannot stop as it has to maintain the wing-borne flight ($v_{\min} > 0$). Next, the dynamics are nonholonomic, meaning that its control towards the lateral direction can be achieved only by changing its direction. Finally, due to small acceleration or deceleration authority, the vehicle often has to employ turning maneuvers for deconfliction. This is common for fixed-wing and hybrid mode vehicles like vertical takeoff-and-landing vehicles (VTOLs), those envisioned for AAM operations \cite{alvarez2019acas, anderson_introduction_flight, FAA_PHAK}. 

Due to these challenges, the advantages of our method for enhancing safety are particularly evident for the air taxi dynamics (Section \ref{subsec:bay-area-sim}). In contrast, for the quadrotors (Section \ref{subsec:crazyflie}), our safety filter design consistently ensures safety across all evaluated methods; thus, we focus more on how our approach achieves performance enhancement. The safe sets, CBVFs, and the potential conflict range computed using HJ reachability are visualized in Figure \ref{fig:conflict-radius}.

\textit{Task \& Training environment.} We modify Multi Particle Environments (MPE) \cite{lowe2017multi} to incorporate agents to follow the dynamics as specified before and the safety filter. \jcnote{In our navigation task setup, the drone must pass through a waypoint with its state satisfying the threshold conditions specified in Table \ref{tab:vehicle_dynamics} to proceed to the next waypoint. The main values that define the training environments are the number of agents $N$, the number of waypoints per agent $M$, and the size of the environment $L$. At every episode, the initial positions of the agents, the waypoints' locations, and the headings are set randomly. The episode is terminated if all agents reach their goal, the last waypoint. After training the MARL policies, we test them in various evaluation scenarios with values of $N, M, L$ different from those of the training environment. For the quadrotor, we use $N$=4, $M$=2, $L$=4, and for the air taxi dynamics, we use $N$=4, $M$=2, $L$=6 for the training.}

\subsection{\jcnote{Comparison Studies}}

\jcnote{We first conduct two sets of simulation experiments to compare our method against: 1) approaches that do not employ a safety filter or curriculum during training, and instead rely on alternative reward designs for safety, and 2) methods from prior works based on model-based CBF design and model-free safe MARL for multi-agent coordination. Both studies are conducted in the quadrotor simulation environment.}

\subsubsection{\jcnote{Ablation study for safety-informed training}} First, we designed our experiments to evaluate the value of (1) introducing the safety filter during training, (2) using the curriculum, and (3) the effectiveness of a potential conflict penalty term for safety, as described in Section \ref{subsec:safety-informed}.
To evaluate the effect of employing the safety filter during training, we compare the results of those trained with and without the filter. To evaluate the effect of the curriculum, we compare our method against a policy trained without the curriculum update in Section \ref{subsec:safety-informed}. Finally, to evaluate the effectiveness of the potential conflict penalty term, we compare it against three alternative penalty terms for safety suggested in the literature:
\begin{itemize}
    \item Hinge loss for constraint violation:
$$\cost_{\textrm{plain}}:=\max\{0, \dsafety - \distance(\srelative)\},$$
This is the most typical penalty term, introduced in the safe RL literature \cite{achiam2017constrained}.
\item CBVF-based hinge loss:
$$\cost_{\textrm{cbvf}}:=\max\{0, - B(\srelative)\},
$$
This penalizes the agent for entering the zero-sublevel set of the CBVF, the unsafe set. The use of reachability value functions as a safety penalty in RL has been explored in previous works such as \cite{asayesh2022least}.
\item Penalty occurring when safety filter intervenes:
$$
\cost_{\textrm{norm.diff}}:=||\aagent_{\textrm{safe}} - \aagent_{\textrm{marl}}||,
$$
based on \eqref{eq:cbvf-qp-single}. This is the main penalty term used in the method of \cite{bejarano2024safety} to inform MARL with safety.
\end{itemize}
When we introduce each penalty term, its weight is carefully tuned to maximize task and safety metrics. In total, we test nine variants of the training methods based on the activation of the safety filter, curriculum, and choice of the reward term, which are detailed in Appendix \ref{appendix:experiments_performed}.

We evaluate each method in three scenarios. In addition to the random scenario same as the training environment, in the second scenario, we test how the MARL policy interacts with a larger number of agents and a more challenging waypoint configuration by setting $N$=6, $M$=3, and $L$=6 while also placing the first two waypoints at the same positions, representing the air corridor. The third scenario reconstructs our hardware experiment environment, which will be detailed in Section \ref{subsec:crazyflie}, in simulation.

Below is the summary of the key aspects of the result, while more details and the table of the comparison for the performance metrics are presented in Appendix \ref{appendix:experiments_performed}:
\begin{itemize}
    \item \textit{Effect of using the safety filter in training:} Methods that incorporate the safety filter during training consistently outperform their counterparts trained without the filter across all metrics.
    \item \textit{Effect of curriculum learning:} The curriculum learning can significantly enhance performance by reducing the conservativeness of the trained policy.
    \item \textit{Effect of potential conflict penalty $\cost_{\textrm{conflict}}$ compared to other penalty candidates:} Our method achieves the best performance in most cases. Importantly, our method outperforms other methods, especially when there is a larger number of agents (the second scenario).
\end{itemize}

\subsubsection{\jcnote{Comparison to other methods}} \jcnote{Next, we compare our method to (1) DG-PPO \cite{zhang2025discrete} and (2) a safety filter designed based on the exponential CBF (ECBF) \cite{nguyen2016exponential}, used for multi-agent collision avoidance in \cite{goarin2024decentralized}. We use $N$=4, $M$=1, $L$=3 for the training of all three methods, which enables fair comparison, especially with results we can obtain from the DG-PPO source code. We run the training of both DG-PPO and our MARL policy with the same number of environment steps (1e7) and gradient steps (epoch\_{ppo}=1), where the initial and goal positions are randomly generated.
}

\jcnote{We evaluate the trained policies in two environments for 25 episodes each: (1) same random environment with world size increased to $L$=6. (2) environment with an increased number of agents $N$=8 and world size $L$=6, where initial and goal positions are arranged in two parallel lines in random order. The results are reported in Tables \ref{table:crazyflie_comparison_small} and \ref{table:crazyflie_comparison_large}. While all methods perform well when the number of agents remains the same as in the training environment, our method is the only method that guarantees 100\% safety when $N$ increases, whereas the percentage of near-collision events increases significantly for DG-PPO and ECBF. It must be noted that DG-PPO is a model-free method that learns a neural CBF during its training, whereas our method uses the CBVF computed based on the system dynamics model. As observed in \cite{zhang2025gcbf}, such model-free methods are vulnerable to generalization in scenarios with a large number of agents.} %

\begin{table}[t]
\caption{Simulation results for Crazyflie dynamics with \textbf{$N$=4}, with time horizon 51.2s. We evaluate goal reach rate (\%) for performance and the percentage of near-collision events ($\distance(\srelative) < \dsafety$) in the timestamped trajectory data (Near collision \%) for safety. }
\label{table:crazyflie_comparison_small}
\vspace{-0.5em}
\centering
\begin{tabular}{|c|cc|}
\hline
\multirow{1}{*}{Methods}
 & Goal reach(\%) & Near collision(\%) \\
\hline
\multirow{1}{*}{DG-PPO} & 96 $\pm$ 11.8 &  0.04 $\pm$ 0.16 \\
\hline
\multirow{1}{*}{Exponential CBF} & 100 $\pm$ 0 & 0.0 $\pm$ 0.0 \\
\hline
\multirow{1}{*}{\textbf{Our Method}} & 100 $\pm$ 0 & 0.0 $\pm$ 0.0 \\
\hline
\end{tabular}
\vspace{-0.8em}
\end{table}

\begin{table}[t]
\caption{Simulation results with \textbf{$N$=8}, and initial \& goal positions arranged in lines under random order. Videos are available in the supplementary material.%
}
\label{table:crazyflie_comparison_large}
\vspace{-0.5em}
\centering
\begin{tabular}{|c|cc|}
\hline
\multirow{1}{*}{Methods}
 & Goal reach(\%) & Near collision(\%) \\
\hline
\multirow{1}{*}{DG-PPO} & 100 $\pm$ 0 &  9.1 $\pm$ 2.7 \\
\hline
\multirow{1}{*}{Exponential CBF} & 93 $\pm$ 8.9 & 8.8 $\pm$ 10.7 \\
\hline
\multirow{1}{*}{\textbf{Our Method}} & \textbf{100 $\pm$ 0} & \textbf{0.0 $\pm$ 0.0} \\
\hline
\end{tabular}\vspace{-1.5em}
\end{table}

\subsection{Hardware experiments with quadrotors}
\label{subsec:crazyflie}

\begin{figure}[t]
    \centering
    \includegraphics[width=0.99\linewidth]{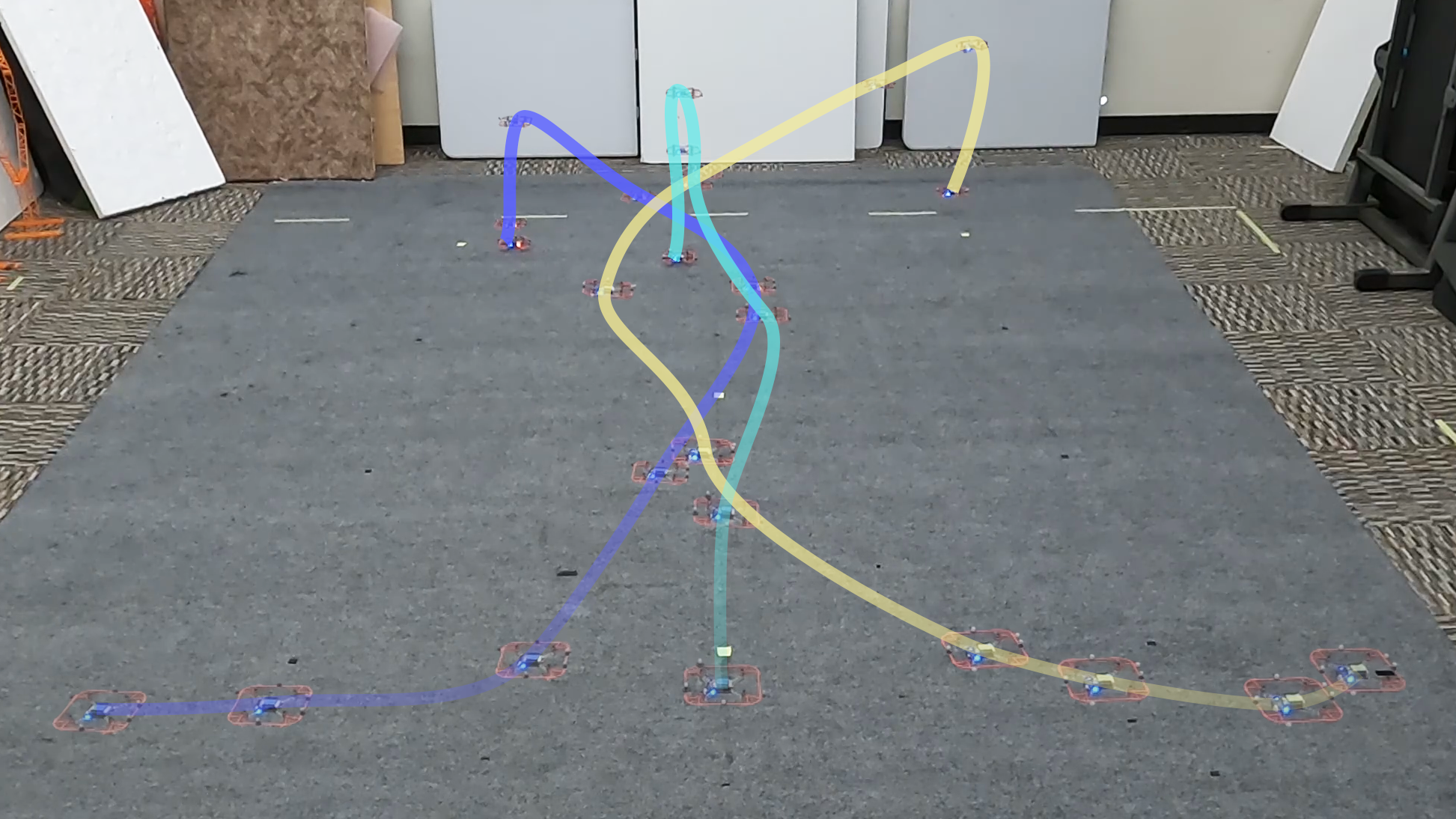}
    \caption{Crazyflie hardware experiment with the MARL policy learned by our method. The three drones have to pass through two common waypoints to get to their landing location. The trajectories corresponding to the video footage are visualized in Fig. \ref{fig:hardware trajectory} (b). }
    \label{fig:hardware picture}
\end{figure}

\begin{figure}[t]
    \centering
    \includegraphics[width=\columnwidth]{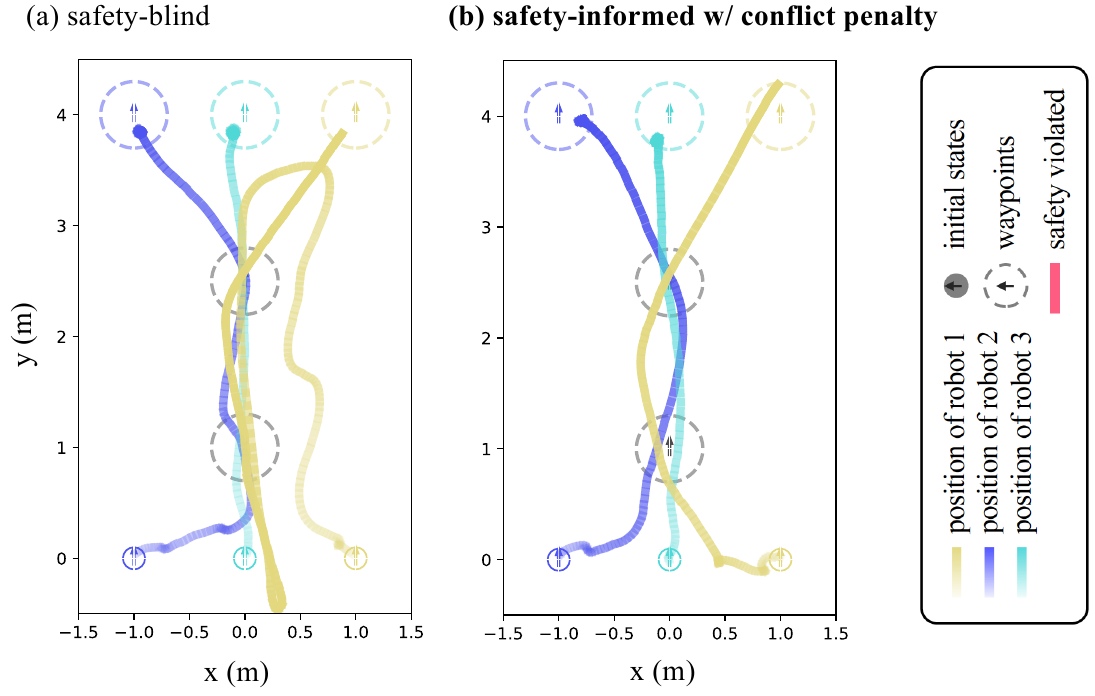}
    \caption{We compare the recorded Crazyflie hardware experiment trajectories under our method and the baseline policy trained without the safety filter. With our approach, the drones smoothly deconflict and efficiently complete the task. In contrast, under the baseline policy, the yellow Crazyflie misses a waypoint and must make a second pass. These results demonstrate that incorporating layered safety information during training improves the performance of the MARL policy.}
    \label{fig:hardware trajectory}
    \vspace{-1em}
\end{figure}

Next, as illustrated in Figure \ref{fig:hardware picture}, we conduct hardware experiments with three Crazyflie 2.0 drones using a Vicon system for localization. Each drone is controlled by a hierarchical framework: its high-level system is modeled with double integrator dynamics, and the resulting high-level state is passed to onboard PID tracking controllers. %
We define high-level feedback control (acceleration in the x and y axes) based on real-time (100 Hz) Vicon system data. To approximate decentralized control, we run distinct decentralized policies for each drone on a single ThinkPad laptop, transmitting high-level control commands every 0.1 seconds.

In the experimental scenario, each drone is required to pass through two shared waypoints—representing an air corridor—before reaching its designated landing location ($N$=3, $M$=3, $L$=3m). We compare the policy learned under our method to the baseline, which is trained without a safety filter in hardware experiments, with the recorded trajectories shown in Figure~\ref{fig:hardware trajectory}. Under our approach, three drones smoothly avoid conflicts and safely navigate their individual waypoints, completing the task in 12.95 seconds. In contrast, the baseline policy requires one drone to perform a second pass, after missing its first waypoint due to the safety filter preventing it from approaching other agents passing the waypoint, thus extending the total completion time to 25.29 seconds. These results demonstrate how our \algoname{} framework enhances task performance through efficient deconfliction. %

\begin{figure}
    \centering
    \includegraphics[width=\columnwidth]{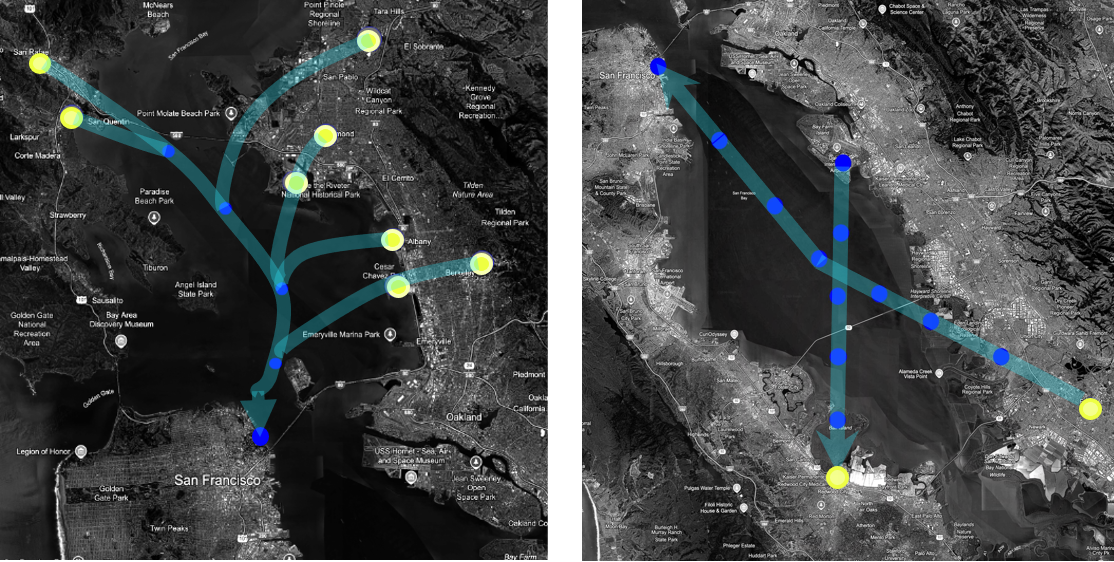}
    \caption{Bay Area case scenarios. The left panel illustrates routes where multiple air taxi vehicles would travel from the North and East Bay toward San Francisco, merging into a single air corridor. The right panel shows intersecting air corridors: one where the vehicles would travel from Fremont (southeast) to San Francisco, and another from Oakland (northeast) to Redwood City. The blue dots represent the waypoints that UAVs follow, while the yellow dots indicate the departure or an incoming waypoint of the corridor.}
    \label{fig:bay-area-map}
    \vspace{-1em}
\end{figure}

\subsection{Simulation of decentralized air taxi operations}
\label{subsec:bay-area-sim}

\begin{figure*}[t!]
    \centering
    \includegraphics[width=\textwidth]{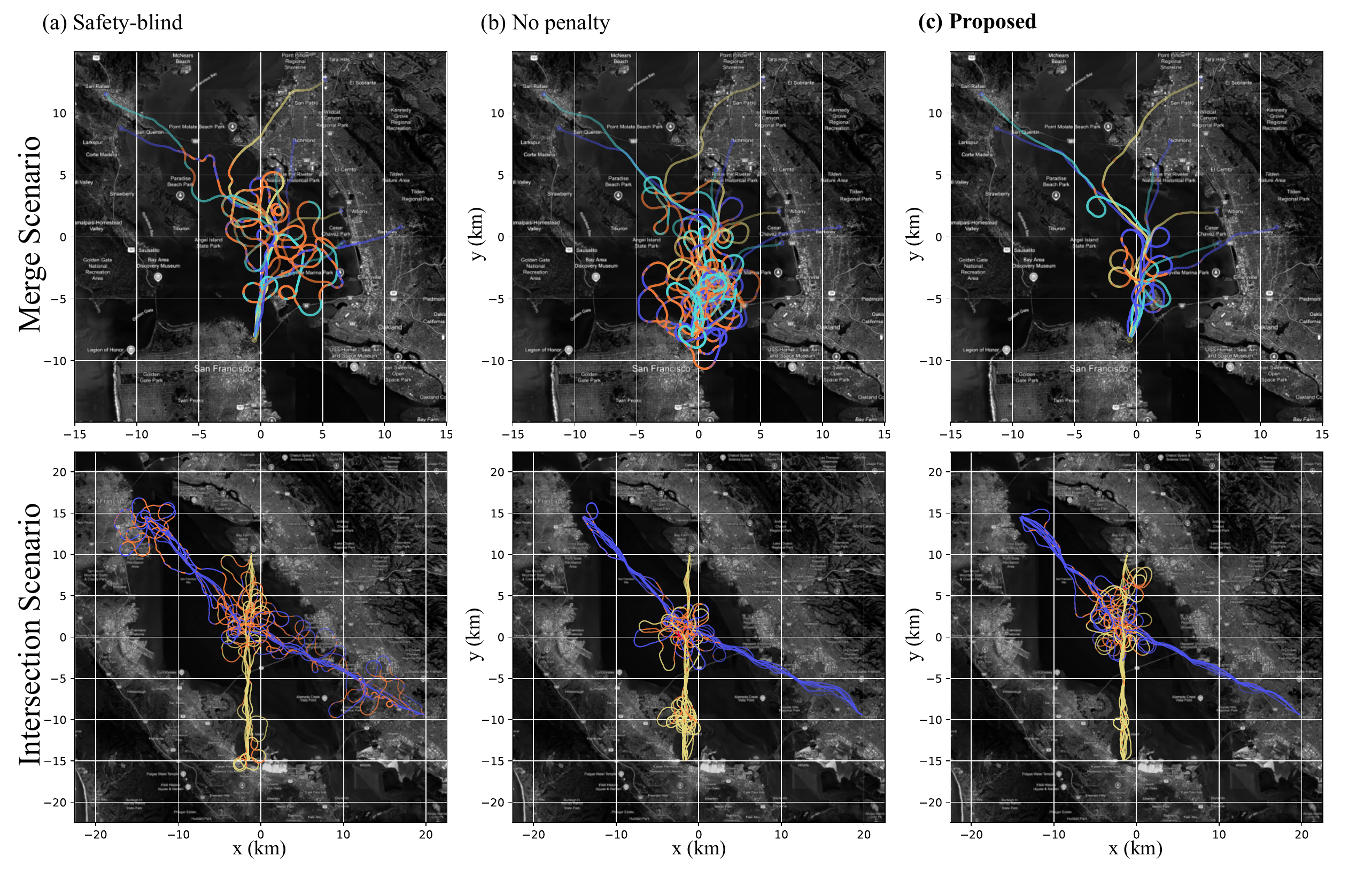}
    \vspace{-1.5em}
    \caption{Comparison of air taxi trajectories in merging and crossing scenarios: The top row illustrates the single-lane merging scenario, where UAVs converge into a shared inbound air corridor, while the bottom row depicts intersecting air corridors. In the merging scenario, our method achieves the most efficient deconfliction of trajectories, minimizing congestion near the corridor. In the crossing scenario, our method demonstrates a wider safety buffer around intersections, as UAVs actively maintain greater separation to mitigate conflicts. Videos are available in the supplementary material.}
    \label{fig:bay-area-main-figure}
\end{figure*}

Finally, we evaluate the application of our framework to decentralized air taxi operations. Although there is no single consensus on how the air traffic management (ATM) system will function for advanced air mobility (AAM) operations, each vehicle will likely be required to have fallback autonomy systems in place, for instance, to ensure safety in case the centralized system fails.

To conduct the study with a realistic traffic volume, we use the results of the Urban Air Mobility (UAM) demand analysis from \cite{Bulusu, Lee}, which estimates how much ground traffic could be replaced by AAM considering various factors including different demographics of riders, socioeconomic factors, and historical commuting patterns. By combining these insights, our study derives a reasonable estimate of how traffic density will evolve once UAM operations reach full implementation. From their results, we consider a peak-density scenario in which each vertiport serves 500 passengers per hour during peak hours, equating to two operations (takeoffs and landings) per minute, with each operation accommodating a 4-passenger aircraft. This corresponds to about 125 trips per hour. We chose vertiports from multiple locations in the Bay Area with high travel demand and designed the air corridors with a lateral separation of 1500 ft based on a preliminary analysis of the separation standards for UAM \cite{Lee}. Waypoints are created to connect the trails of these corridors spaced 3-4 km apart. 
For simplicity and clearer visualization, our simulations focus on aerial vehicles traveling westward (from the East Bay to San Francisco and to the South Bay); we assume outbound trips use a separate fixed altitude, thus our study addresses only horizontal deconfliction. %

An important modeling assumption is the required separation distance between vehicles. We base these restrictions on industry standards that define the minimum safe distance between the aircraft and potential hazards to maintain an acceptable collision risk \cite{Shish}. Although there is no single global standard for the separation distance for AAM vehicles yet, we adopt a set of parameters from a variety of literature: the separation distance from dynamic obstacles, maximum accelerations, and bank angle to ensure safety and passenger comfort, \cite{FAA_NMAC, Shish, Lee}, which guides the minimum horizontal distances between UAVs to range from 500 to 2200 feet. In our study, a horizontal separation of 1500 ft was used, as it aligns with NASA's UAM corridor design and provides a good balance between collision risk and operational efficiency. %

\begin{table}[t]
\caption{Simulation results of air taxi operations emulating potential peak traffic around the Bay Area—a scenario in which all vehicles merge into the city-inbound corridor. For performance, we evaluate the mean travel time (s). For safety, we evaluate the percentage of near-collision events ($\distance(\srelative) < \dsafety$) in the timestamped trajectory data (Near collision \%), and the percentage of instances having multiple agents encountered within the potential conflict range ($\distance(\srelative) < \dengagement$) (Conflict \%).}
\centering
\begin{tabular}{|c|ccc|}
\hline
\multirow{2}{*}{Methods} & \multicolumn{3}{c|}{Merging Scenario ($N$=8, $M$=5)} \\
\cline{2-4}
 & Travel t(s)($\downarrow$) & Near collision(\%)($\downarrow$) & Conflict(\%)($\downarrow$) \\
\hline
\multirow{1}{*}{Safety-blind} & 675.6 & 0.055 & 2.4 \\
\hline
\multirow{1}{*}{No penalty} & 617.9 & 0.042 & 5.5\\
\hline
\multirow{1}{*}{\textbf{Proposed}} & \textbf{450.5} & \textbf{0.021} & \textbf{3.2}\\
\hline
\end{tabular}
\label{table:airtaxi-scenario1}
\end{table}

\begin{table}[!tb]
\caption{Simulation results of air taxi operations---a scenario in which two air corridors intersect with each other.}
\centering
\begin{tabular}{|c|ccc|}
\hline
\multirow{2}{*}{Methods} & \multicolumn{3}{c|}{Intersection Scenario ($N$=16, $M$=6)} \\
\cline{2-4}
 & Travel t(s)($\downarrow$) & Near collision(\%)($\downarrow$) & Conflict(\%)($\downarrow$) \\
\hline
\multirow{1}{*}{Safety blind} & 987.4 & 0.058 & 2.1 \\
\hline
\multirow{1}{*}{No penalty} & 780.5 & 0.129 & 3.8 \\
\hline
\multirow{1}{*}{\textbf{Proposed}} & \textbf{660.8} & \textbf{0.056} & \textbf{1.6} \\
\hline
\end{tabular}
\label{table:airtaxi-scenario2}
\vspace{-1em}
\end{table}

We evaluate three methods: MARL trained without the safety filter and no safety penalty (\textbf{Safety blind}), MARL trained with the safety filter under the proposed curriculum but with no safety penalty (\textbf{No penalty}), and the proposed safety-informed method employing the safety filter, curriculum, and the potential conflict penalty (\textbf{Proposed}) in two high-density air traffic scenarios. The two scenarios, illustrated in Figure \ref{fig:bay-area-map}, represent different commuting patterns in the Bay Area. In the left scenario (\textbf{Merge Scenario}), multiple air routes (eight in total) from the northern Bay merge into a single corridor leading to San Francisco. The departure time of each vehicle varies randomly within a 60-second range. In the scenario shown to the right (\textbf{Intersection Scenario}), two westbound air corridors intersect. Here, we set the UAVs to leave the origin every $90 \pm 15 $ seconds to intentionally induce congestion at the intersection. 

We evaluate 25 random episodes for each method and report the results in Table \ref{table:airtaxi-scenario1} and Table \ref{table:airtaxi-scenario2} for each scenario, respectively. The results show that the proposed method achieves both the highest performance, measured by the shortest mean travel time, and the lowest percentage of near-collision events.

Examples of vehicle trajectories for each scenario and method are visualized in Figure \ref{fig:bay-area-main-figure}. In the merging scenario, our method demonstrates the most efficient deconfliction of trajectories when multiple vehicles merge into the air corridor. In the intersection scenario, near the intersection, the region occupied by vehicles as they maneuver to avoid collisions is noticeably larger in our method compared to the second method, trained with no penalty. This indicates that vehicles using our approach proactively maintain greater separation to mitigate the conflicting constraints.

While we expect that a fully operational ATM system for AAM will be significantly more efficient, seamless, and safer than our simulation study suggests, we present this study as an initial guideline for resolving hypothetical emergency scenarios. For instance, the situations we simulated can emerge when the airspace congestion coincides with the loss of centralized traffic control, requiring each agent to make independent, safe decisions.

\section{Limitations}
\label{sec:limitations}
While our framework shows significant improvements in achieving safety and performance, there are limitations, which we list below.
\begin{enumerate}
    \item \textit{Scalability to higher dimensions}: Our current framework is designed for 2D scenarios, and needs extending it to 3D environments.
    \item \textit{Guarantees for multiple engagements}: Our safety guarantees are currently limited to pairwise interactions. While our method is designed to significantly mitigate collision risks in multi-agent interactions based on theoretical analysis, it does not provide formal guarantees for scenarios involving multi-agent engagements. For further details, see Remark \ref{remark:theory-limitation}.
    \item \textit{Hardware experimentation constraints}: In our hardware experiments, local observations were emulated using a motion capture (mocap) system rather than being obtained through onboard sensing. This simplification may not fully reflect real-world operational constraints and should be addressed in future implementations. Additionally, fixed-wing vehicles were not included in our experiments.
    \item \textit{Communication range constraints}: The impact of communication range limitations in our hardware experiments was not analyzed.
\end{enumerate}

\begin{table*}[!htb]
\caption{Results of policies trained under various methods for Crazyflie dynamics: We evaluate mean travel time (s) and number of reached waypoints (Waypoint \#) for performance, and the percentage of the events involving multiple agents encountered within the potential conflict range in the trajectory data (Conflict \%) for safety risk. Note that in these simulations, the agent never violated safety for all methods due to our safety filter, except in the training scenario when the agent is initialized at the safety-violating states. ($N$=number of agents, $M$=number of waypoints, $L$=world size)}
\centering
\begin{tabular}{|c|ccc|ccc|ccc|}
\hline
\multirow{2}{*}{Methods} & \multicolumn{3}{c|}{Scenario 1 (Training) ($N$=4, $M$=2, $L=4$)} & \multicolumn{3}{c|}{Scenario 2 ($N$=6, $M$=3, $L=6$)} & \multicolumn{3}{c|}{Scenario 3 ($N$=3, $M$=3, $L=3$)} \\
\cline{2-10}
 & Travel time(s)($\downarrow$) & Waypoint\#($\uparrow$) & Conflict(\%)($\downarrow$)
 & Travel t & Waypoint \# & Conflict
 & Travel t & Waypoint \# & Conflict \\
\hline
\multirow{1}{*}{1} (safety-blind) & $18.33$ & $1.62 \pm 0.25$ & 8.7 
  & $29.18$ & $2.11 \pm 0.24$ & 21.9 
  & $19.08$ & $2.14 \pm 0.53$ & 16.2 \\
\hline
\multirow{1}{*}{2} & $18.21$ & $1.67 \pm 0.19$ & 7.1 
  & 29.77 & $2.06 \pm 0.24$ &  19.7
  & $18.92$ & $2.40 \pm 0.47$ & 16.0 \\
\hline
\multirow{1}{*}{3} & $17.73$ & $1.76 \pm 0.16$ & 6.6 
  & $28.59$ & $2.26 \pm 0.21$ & 19.7 
  & $18.44$ & $2.41 \pm 0.46$ & 14.9 \\
\hline
\multirow{1}{*}{4} & $18.73$ & $1.58 \pm 0.23$ & 7.2 
  & $29.11$ & $2.20 \pm 0.21$ & 17.1 
  & $18.79$ & $2.17 \pm 0.47$ & 13.4 \\
\hline
\multirow{1}{*}{5} (no penalty) & $17.56$ & $1.75 \pm 0.17$ & 5.3
  & $28.46$ & $2.33 \pm 0.20$ & 16.9 
  & $16.09$ & $2.78 \pm 0.28$ & 11.8 \\
\hline
\multirow{1}{*}{6} & $18.31$ & $1.67 \pm 0.17$ & 5.9 
  & $28.92$ & $2.27 \pm 0.21$ & 17.0 
  & $17.69$ & $2.42 \pm 0.42$ & 15.0 \\
\hline
\multirow{1}{*}{7} & $17.90$ & $1.71 \pm 0.18$ & 5.3 
  & $28.98$ & $2.32 \pm 0.18$ & 15.8 
  & $17.59$ & $2.54 \pm 0.34$ & 11.2 \\
\hline
\multirow{1}{*}{8} & $20.52$ & $1.26 \pm 0.21$ & 2.7
  & $31.25$ & $1.74 \pm 0.27$ & 6.9 
  & $18.36$ & $2.20 \pm 0.36$ & 8.6 \\
\hline
\multirow{1}{*}{\textbf{9 (proposed)}} & \textbf{17.81} & \textbf{1.78} $\pm$ \textbf{0.18} & \textbf{5.4} 
  & \textbf{28.59} & \textbf{2.42} $\pm$ \textbf{0.20} & \textbf{15.1}
  & \textbf{16.91} & \textbf{2.71} $\pm$ \textbf{0.24} & \textbf{10.8} \\
\hline
\end{tabular}
\label{table:crazyflie-sim}
\end{table*}

\section{Conclusions}
\label{sec:conclusion}
In this work, we presented a layered architecture combining a CBVF-based safety filtering mechanism with a MARL policy, demonstrating its effectiveness in ensuring both safety and efficiency. Our approach enables MARL to navigate conflicts proactively while benefiting from safety-informed reward signals. Along with the safety filter introduced during training using a curriculum learning approach, the \algoname{} framework achieved shorter travel times and reached more waypoints with fewer conflicts. The key components of our approach, curriculum learning, and cost terms that inform potential conflict zones are agnostic to the choice of the MARL algorithm. We validated our method by applying it to two distinct dynamics---quadrotor and fixed-wing AAM flight dynamics---and evaluated it in progressively complex scenarios. We also conducted hardware experiments on three Crazyflie drones, highlighting the applicability of our method in real-world systems. %

\jcnote{Our method integrates model-based safety tools from control theory (CBVFs) with learning-based methods (MARL), together forming a framework that addresses two major challenges in multi-agent problems---safety and efficient coordination. While deep reinforcement learning has faced skepticism in safety-critical applications such as air traffic management, recent advances---including our work---demonstrate the viability of hybrid approaches that combine learning and control, and illustrate how RL can be responsibly applied in safety-critical settings.
}

Future research could investigate decomposition techniques and learning-based reachability analysis (e.g., DeepReach \cite{bansal2021deepreach}) to extend safety verification to higher-dimensional settings. Adapting to other methods, including other MARL algorithms (e.g., MAPPO or even further refining DG-PPO), MPC, or game-theoretic solutions, by treating the potential conflict zone as a soft constraint, is an exciting future work direction. Further investigation is needed to assess how communication constraints affect coordination and safety in decentralized multi-agent systems. Finally, an important future direction is testing the proposed approach's applicability in various robotics domains, ranging from higher-order dynamics to complex environments and sensing constraints, such as ground robots, underwater autonomous vehicles, and space robots.

\section*{Acknowledgments}
We thank Dr. Mir Abbas Jalali (Joby), George Gorospe (NASA), Dr. Anthony Evans (Airbus), Inkyu Jang (SNU) and Kanghyun Ryu (UC Berkeley) for the helpful discussions. 
Jasmine Aloor and Hamsa Balakrishnan were supported in part by NASA  under Grant No. 80NSSC23M0220. Jason J. Choi, Jingqi Li and Claire J. Tomlin were supported
in part by DARPA Assured Autonomy under Grant FA8750-18-C-0101, DARPA ANSR under Grant FA8750-23-C-0080, the NASA ULI
Program in Safe Aviation Autonomy under Grant 62508787-176172, and ONR Basic Research Challenge in Multibody Control Systems under Grant N00014-18-1-2214. 
Jason J. Choi received the support of a fellowship from Kwanjeong Educational Foundation, Korea. 
Jasmine J. Aloor was supported in part by a Mathworks Fellowship. Maria G. Mendoza acknowledges support from NASA under the Clean Sheet Airspace Operating Design project MFRA2018-S-0471. The authors would like to thank the MIT SuperCloud \cite{supercloud} and the Lincoln Laboratory Supercomputing Center for providing high performance computing resources that have contributed to the research results reported within this paper.

\appendix

\subsection{Reward function design for goal reaching}
\label{appendix-goal-reward}

We assume that, from the agent's local observation, it can evaluate its state, including the position and heading relative to the current target waypoint, denoted as $s_{\mathrm{ref}}^{(i)}$. The additional reward ${\mathcal{R}}_\mathrm{tracking}$ is designed to guide agents to efficiently navigate to the target waypoint. It uses a distance-like measure relative to the waypoint position, and can incorporate additional information like the errors from the desired heading angle and speed. The vehicle dynamics also inform the reward design. We design a specific reward for each quadrotor and air taxi dynamics. For the quadrotor, although the vehicle dynamics are holonomic, we want the vehicle to approach the waypoint from a specific target heading direction. To achieve this, we design a reference velocity field, $v_{\mathrm{ref}}(s_{\mathrm{ref}}^{(i)})$, around the waypoint, shaping it like the magnetic field around a solenoid. Then, the reward for the waypoint tracking is given as
$
{\mathcal{R}}_\mathrm{tracking}(\oagent_\timestep, a^{(i)}_\timestep) =  ||v^{(i)} - v_{\mathrm{ref}} ||
$. For the air taxi dynamics, shaping the reference velocity field is not straightforward, as the vehicle is nonholonomic and is mainly limited by its maximum turning radius, determined by its speed and yaw rate. Thus, we compute the time-to-reach function \cite{yang2013one}, which is the minimum time required to reach the target waypoint (satisfying the threshold conditions) subject to the vehicle dynamics. This time-to-reach reward is also used in \cite{lyu2020ttr} for RL-based navigation of mobile robots. The use of the time-to-reach reward guides the vehicle to learn how to perform a 360-degree turn when it misses its waypoint.

\subsection{Air taxi dynamics: additional information}
\label{appendix-airtaxi-dynamics}
The relationship between $\srelative$ and ($\sagent$, $\sagentj$) for the air taxi dynamics in \eqref{eq:airtaxi} is given by
\begin{equation}
\begin{aligned}
x^{(ij)} & = \cos \theta^{(ij)} (x^{(j)} - x^{(i)}) + \sin \theta^{(ij)} (y^{(j)} - y^{(i)}), \\
y^{(ij)} & = -\sin \theta^{(ij)} (x^{(j)} - x^{(i)}) + \cos \theta^{(ij)} (y^{(j)} - y^{(i)}), \\
\theta^{(ij)} & = \theta^{(i)} - \theta^{(j)}.
\end{aligned}
\label{eq:airtaxi-relative}
\end{equation}
The relative dynamics \eqref{eq:relative} can be derived from \eqref{eq:airtaxi} and \eqref{eq:airtaxi-relative}, and are express as
\begin{equation}
\begin{aligned}    
\dot{x}^{(ij)} &= - v^{(i)} + v^{(j)} \cos \theta^{(ij)}  + y^{(ij)} \omega^{(i)} \\
\dot{y}^{(ij)} &= v^{(j)} \sin \theta^{(ij)} - x^{(ij)} \omega^{(i)}\\
\dot{\theta}^{(ij)} &= \omega^{(j)} - \omega^{(i)}.
\end{aligned}
\end{equation}

\subsection{Ablation Study: Details}\label{appendix:experiments_performed}

We conduct comparison studies among the following nine methods:
\begin{enumerate}
    \item Policy trained without the safety filter and no safety penalty \textbf{(Safety blind)}
    \item Policy trained without safety filter and with $\cost_{\textrm{plain}}$
    \item Policy trained without safety filter and with $\cost_{\textrm{conflict}}$
    \item Policy trained with the safety filter and without curriculum learning (with no penalty)
    \item Policy trained with the safety filter and no safety penalty \textbf{(No penalty)}
    \item Policy trained with the safety filter and with $\cost_{\textrm{plain}}$
    \item Policy trained with the safety filter and with $\cost_{\textrm{cbvf}}$
    \item Policy trained with the safety filter and with $\cost_{\textrm{norm.diff}}$
    \item \textbf{Policy trained with the safety filter and with $\cost_{\textrm{conflict}}$ (Proposed)}
\end{enumerate}
Note that methods 5-9 are trained with curriculum learning on $\dsafety$, as described in Section \ref{subsec:safety-informed}. Every policy we compared has been trained for the same number of environment steps. Methods 1, 5, and 9 correspond to the methods we also evaluate in the air taxi operation simulation.

Table \ref{table:crazyflie-sim} summarizes the results of the simulation study, and Figure~\ref{fig:result-circular-config-di} visualizes example trajectories in Scenario 2 under the policies of methods 1, 5, and 9. Each method is evaluated using four random seeds, with 25 episodes per seed, totaling 100 random episodes. Note that in these simulations, the agent never violated safety for all methods due to our safety filter, except in the training scenario when the agent is initialized at the safety-violating states. Thus, the percentage of near-collision events (safety violation) is not reported, and only the rate of potential conflict (for instance, when more than two agents enter the potential conflict range $\dengagement$) is calculated.

The key aspects of the result in Table \ref{table:crazyflie-sim} are:
\begin{itemize}
    \item \textit{Effect of using the safety filter in training (1-3 vs 5-6, 9):} Methods that incorporate the safety filter during training consistently outperform their counterparts trained without the filter across all metrics.
    \item \textit{Effect of curriculum learning (4 vs 5):} The curriculum learning can significantly enhance performance by reducing the conservativeness of the policy.
    \item \textit{Effect of potential conflict penalty $\cost_{\textrm{conflict}}$ compared to other penalty candidates (6, 7, 8 vs 9):} Although method 8 that uses $\cost_{\textrm{norm.diff}}$ consistently shows the lowest rate of potential conflict, and its performance is significantly impaired by the penalty. Our method achieves the best performance in most cases. Importantly, our method outperforms other methods, especially when there is a larger number of agents (Scenario 2).
\end{itemize}

\begin{figure}[!t]
    \centering
    \includegraphics[width=0.7\columnwidth]{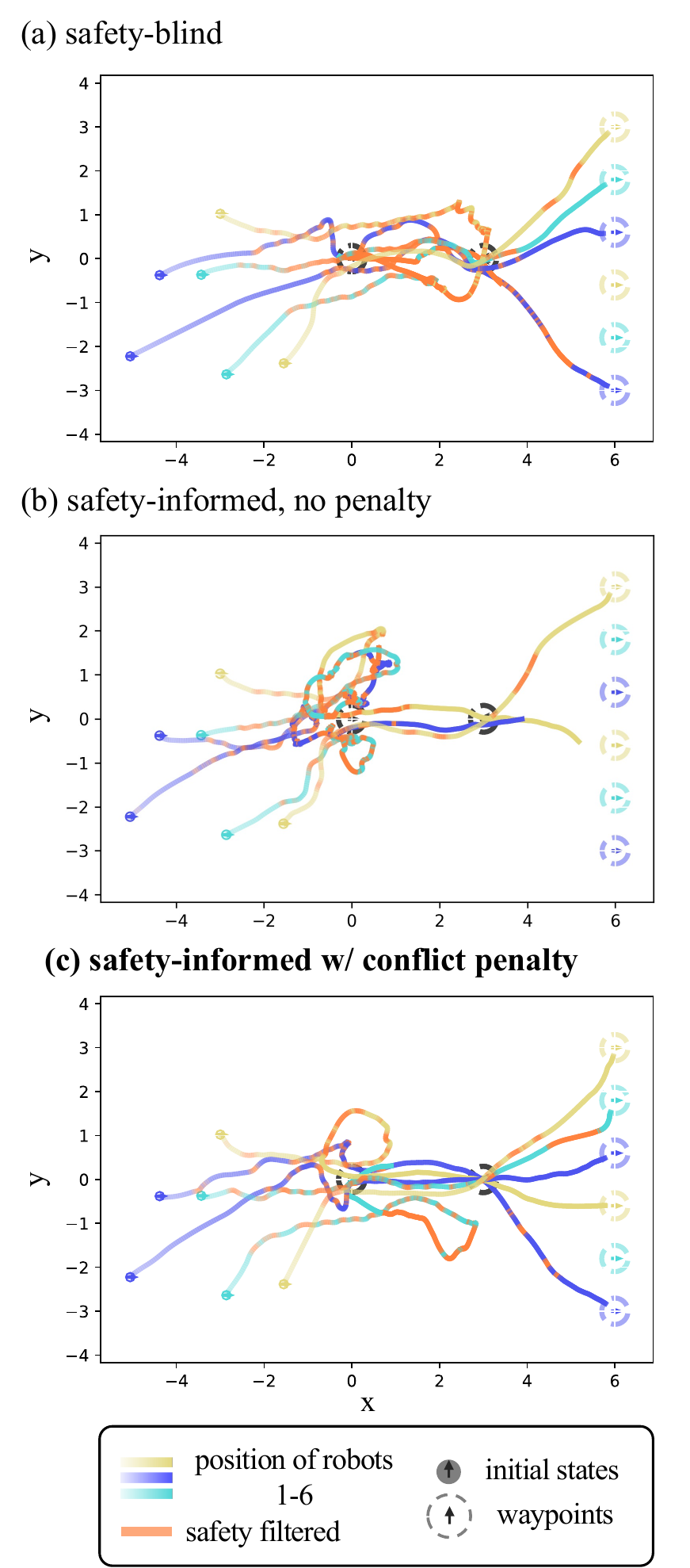}
    \caption{Simulation results of (a) safety-blind (method 1), (b) safety-informed with no penalty (method 5), and (c) safety-informed with potential conflict penalty (method 9) under Scenario 2 in Table \ref{table:crazyflie-sim}, trained for double integrator dynamics. Agents are initialized at random positions and have to merge into a line formed by two waypoints before reaching their final waypoints. While our safety filter ensures safety for all cases, the MARL method trained with a potential conflict penalty shows the most efficient behavior for reaching waypoints. Videos are available in the supplementary material.}
    \label{fig:result-circular-config-di}
\end{figure}

\;

\balance

\bibliographystyle{unsrtnat}
\bibliography{references}

\end{document}